\documentclass[appendixprefix=Appendix, 11pt]{article}
\usepackage[utf8]{inputenc}

\usepackage[margin=1in]{geometry}
\usepackage{amsmath, amssymb, bm, xcolor, graphics,bbm, bm, nicefrac, centernot, multicol,todonotes,lipsum,caption,subcaption,dsfont,amsthm,authblk,lmodern, setspace}
\usepackage[ruled]{algorithm2e}
\usepackage[natbib=true, style=authoryear]{biblatex}
\usepackage[colorlinks=true,
            breaklinks=true,
            bookmarks=true,
            urlcolor=blue,
            citecolor=blue,
            linkcolor=blue,
            bookmarksopen=false,
            draft=false]{hyperref}
\def\EMAIL#1{\href{mailto:#1}{#1}} 

\usepackage[title]{appendix}

\DeclareMathOperator*{\argmin}{argmin}
\DeclareMathOperator*{\argmax}{argmax}

\DeclareMathOperator{\dreg}{\mathrm{D-Reg}\xspace}
\newcommand{\deltaIk}{\Delta_{\infty}^K\xspace}
\newcommand{\deltaI}{\Delta_{\infty}\xspace}

\usepackage{thm-restate}

\newtheorem{definition}{Definition}

\newtheorem{claim}{Claim}
\newtheorem{remark}{Remark}

\title{MNL-Bandit in non-stationary environments}
\author{
  \vspace{3mm}
  Ayoub Foussoul\\
  \vspace{-4mm}
  {\footnotesize Industrial Engineering and Operations Research, Columbia University, New York, USA, 
  \EMAIL{af3209@columbia.edu}}
  \and
  \vspace{-3mm}
  Vineet Goyal\\
  \vspace{-4mm}
  {\footnotesize Industrial Engineering and Operations Research, Columbia University, New York, USA, 
  \EMAIL{vg2277@columbia.edu}}
  \and
  \vspace{-3mm}
  Varun Gupta\\
  \vspace{-4mm}
  {\footnotesize Booth School of Business, University of Chicago, Chicago, USA, 
  \EMAIL{varun.gupta@chicagobooth.edu}}
}
\date{}

\begin{document}

\maketitle

\begin{abstract}%
\noindent In this paper, we study the MNL-Bandit problem in a non-stationary environment and present an algorithm with a worst-case expected regret of $\Tilde{O}\left( \min \left\{ \sqrt{NTL}\;,\; N^{\frac{1}{3}}(\deltaIk)^{\frac{1}{3}} T^{\frac{2}{3}} + \sqrt{NT}\right\}\right)$. Here $N$ is the number of arms, $L$ is the number of changes and $\deltaIk$ is a variation measure of the unknown parameters. Furthermore, we show  matching lower bounds  on the expected regret (up to logarithmic factors), implying that our algorithm is optimal. Our approach builds upon the epoch-based algorithm for stationary MNL-Bandit in \cite{mnlucb}. However, non-stationarity poses several challenges and we introduce new techniques and ideas to address these. In particular, we give a tight characterization for the bias introduced in the estimators due to non stationarity and derive new concentration bounds.
\end{abstract}


\section{Introduction}

The MNL-Bandit problem is a combinatorial variant of the traditional stochastic multi-armed Bandit (MAB) problem. In this problem, we are given a set of $N$ arms with known payoffs, $r_1,\ldots,r_N$ and unknown parameters $\omega_1, \ldots, \omega_N$. At each time step, the decision maker selects a subset, $S$ of at most $K$ arms. Then Nature randomly selects arm $i \in S$ with probability that follows a multinomial logit (MNL) model and is given by $\omega_i/(1+ \sum_{j \in S} \omega_j)$.
The decision-maker receives a payoff $r_i$ corresponding to arm $i$. Nature selects no arm at all with probability $1/(1+\sum_{j\in S} \omega_j)$ in which case the decision maker receives a payoff of $0$. The parameters $\omega_1, \ldots, \omega_N$ are referred to as the {\em attraction parameters} of the MNL model.

The MNL-bandit problem arises commonly in an online retail setting where $N$ arms correspond to $N$ substitutable products with known prices. The customers arrive sequentially and choose according to the same MNL model with unknown attraction parameters. For each customer, the seller offers a subset of at most $K$ products and collects the payoff corresponding to the random choice of the customer from the offered set (or possibly no payoff if no purchase happened) according to the MNL model. The goal of the seller is to maximize the total expected payoff over $T$ customers.

MNL-Bandit is a well studied problem. \cite{mnlucb} consider the MNL-bandit problem in a setting where the attraction parameters are stationary and give a UCB algorithm that achieves a regret of $\Tilde{O}(\sqrt{NT})$. This  is optimal up to logarithmic factors \citep{chen2017note}. In particular, they present an epoch based algorithm that allows estimation of the attraction parameters via sampling; thereby, overcoming the challenge of a non-linear expected payoff function.

In many applications however, the MNL parameters $\omega_1, \dots, \omega_N$ change over the horizon of interest. In online retailing for instance, the attraction parameters of the products typically vary over the selling horizon. Motivated by these settings, we study in this paper the MNL-Bandit problem under non-stationary parameters. In particular, we consider an adversarial model of non-stationarity where the attraction parameters are chosen adversarially, and develop an online learning algorithm minimizing the worst-case {\em dynamic regret} of the decision maker (the gap between the total recovered payoff and the total payoff from taking the optimal action at each time step). We give regret bounds that are in function of the number of changes (or switches) in the attraction parameters, and of a variation parameter characterizing the total amount of change in the attraction parameters. This follows a series of works on reinforcement learning problems in non-stationary environments where a similar model of non-stationarity is considered for a variety of learning problems including the multi-armed bandit problem \citep{auer2002nonstochastic,garivier2011upper,besbes2014stochastic,auer2019adaptively,wei2021non}, contextual bandits \citep{luo2018efficient,chen2019new}, linear bandits \citep{cheung2019learning,russac2019weighted,zhao2020simple} and more \citep{hazan2009efficient,yang2016tracking,russac2020algorithms,mao2020model,touati2020efficient}.

A widely used approach to handle non-stationarity in bandit problems is to run an algorithm with good performance in stationary environments (that we refer to as a base algorithm), and keep testing for non-stationarity. Once non-stationarity is detected, we restart the base algorithm. The dynamic regret from this approach depends on the regret that the base algorithm accumulates in the ``near-stationary" environment before non-stationarity is detected. In a classical setting of $n$-arm multi-armed bandit problem, the base algorithm computes an unbiased estimator of the unknown parameters in each period and evaluating the regret accumulated by the base algorithm in a near-stationary environment is relatively easy. For the MNL-bandit problem where the expected reward function is non-linear, even in  a stationary setting, we do not know of an algorithm with good regret guarantees that computes unbiased estimators of the attraction parameters in each period. \cite{mnlucb} give an optimal algorithm for MNL-bandit problem in a stationary setting. They consider an epoch based algorithm that uses multiple time steps (an epoch) to get a single unbiased estimate of the attraction parameters. However, in a near-stationary environment where the attraction parameters vary with time, evaluating the regret accumulated by an epoch based approach becomes significantly more challenging and requires new analysis techniques.

\subsection{Our contributions}

We give an algorithm for the MNL-bandit problem in non-stationary environments with worst-case dynamic regret of $\tilde{O}\left( \min \left\{ \sqrt{NTL}\;,\; N^{\frac{1}{3}}(\deltaIk)^{\frac{1}{3}} T^{\frac{2}{3}} + \sqrt{NT}\right\}\right).$ Here $L$ is the number of switches defined formally as $
L := 1+\sum_{t=1}^{T-1} \bm{1}\{\bm{\omega}(t) \neq \bm{\omega}(t+1)\}
$, where $\bm{\omega}(t)$ denotes the vector of attraction parameters at round $t$. Also $\deltaIk$ is the variation in norm $L^{2K}_\infty$ defined as $
\deltaIk := \sum_{t=1}^{T-1} \lVert \bm{\omega}(t) - \bm{\omega}(t+1) \rVert_{L^{2K}_\infty},
$ where $\lVert x \rVert_{L^{2K}_\infty} = \sup_{S\subset [N]:|S| \leq 2K} \sum_{r \in S} |x_r|$ for $x \in \mathbb{R}^N$. We refer to this algorithm as Exploration-Exploitation algorithm for MNL-bandit in Non-Stationary Environments (EEMNSE). EEMNSE consists of an adaptation of the UCB-based algorithm of \cite{mnlucb} to near-stationary environments (see Algorithm~\ref{alg:near-stationary}), combined with the algorithm MASTER of \cite{wei2021non} which is a meta algorithm that takes in input a base algorithm (here Algorithm~\ref{alg:near-stationary}), schedules this algorithm in a multi-scaled manner and restarts when non-stationarity of the environment is detected.

\paragraph{Challenges and key technical contributions.} 

One of the main difficulties in the MNL-bandit problem arises from non-linearity of the expected payoff function even in a stationary environment.
\cite{mnlucb} overcome this by developing an epoch based approach where the same assortment is offered for multiple consecutive periods (referred to as an epoch) until a no-purchase is observed.  \cite{mnlucb} show that the number of times an arm is chosen within an epoch is an unbiased estimator of the attraction parameter of the corresponding arm and has a geometric distribution \footnote{We would like to remark that the distribution of the estimators is not exactly geometric due to the truncation from a finite time horizon. We address this technical issue in our analysis.}. However, this is not true in a non-stationary environment. In particular, because the attraction parameters can change within an epoch, this estimator is biased and its distribution depends on the adversarial changes to the parameters. The main challenge consists of getting a tight characterization of the bias caused by non-stationarity.

We introduce a new stochastic process that we refer to as the {\em uninterrupted purchases process} (UPP) that we use to get a tight characterization of the bias. In particular, we sandwich UPP (in the usual stochastic order) between two geometric random variables with bounded difference of means. To analyse the regret, we show new concentration bounds for (UPP) by leveraging our stochastic ordering. Specifically, we present non-symmetric concentration bounds for (UPP) in the sense that one of the bounds is with respect to the mean of the stochastic upper-bounds and the other is with respect to the mean of the stochastic lower-bounds. We use our concentration bounds to evaluate the gap between the attraction parameters and their estimators. Finally, we translate this gap into a gap in the payoff function using a special Lipschitz property of the payoff function.
Our approach might be of independent interest in dealing with non-stationarity in epoch-based algorithms.

\paragraph{Matching lower-bounds.} We also prove lower-bounds for the MNL-Bandit problem in non-stationary environments with respect to the number of switches $L$ and a novel variation measure $\deltaIk$ respectively. In particular, we show that any polynomial time algorithm for the MNL-Bandit problem in a non-stationary environment with $L$ switches must incur a regret of at least $\Tilde{\Omega}\left( \sqrt{NLT}\right)$. Furthermore, in a non-stationary environment with variation $\deltaIk$, any algorithm must incur a regret $\Tilde{\Omega}\left( (N)^{\frac{1}{3}}(\deltaIk)^{\frac{1}{3}} T^{\frac{2}{3}}\right)$. Both of these bounds match our upper-bound for our online algorithm up to logarithmic factors. Therefore, our algorithm is near-optimal (up to logarithmic factors) for the MNL-Bandit problem in non-stationary environments.

We also show that the variation measure $\deltaIk$ we introduce which is different from the variation in norm infinity $L_{\infty}$ (that we denote by $\deltaI$) that is commonly used in other bandit settings, provides a strictly better characterization of the regret. In particular, we show that any polynomial time algorithm must incur a regret of $\Tilde{\Omega}\left( (KN)^{\frac{1}{3}}(\deltaI)^{\frac{1}{3}} T^{\frac{2}{3}}\right)$ in the worst case (note that this bound is acheived by our algorithm). Since $\deltaIk \leq 2K \cdot \deltaI$, the upper-bound $\Tilde{O}\left( N^{\frac{1}{3}}(\deltaIk)^{\frac{1}{3}} T^{\frac{2}{3}}\right)$ is a strictly better characterization of the performance of our algorithm than $\Tilde{O}\left( (NK)^{\frac{1}{3}}(\deltaI)^{\frac{1}{3}} T^{\frac{2}{3}}\right)$.

\subsection{Other related literature}

MNL-Bandit has been first considered by \cite{Rusmevichienton}, where the authors develop an ``Explore-Then-Commit" algorithm with worst-case regret $O(N^2\log^2 T)$, assuming the knowledge of the ``gap" between the optimal and next-best assortment. This bound has been later improved to $O(N\log T)$ by \cite{saurezeevi} under similar assumptions. \cite{mnlucb} were the first to develop a parameter-free algorithm for MNL-bandit. Their algorithm achieves a worst-case regret of $\Tilde{O}(\sqrt{NT})$. They also prove a lower bound of $\Omega(\sqrt{NT \slash K})$ for the problem which has been subsequently improved to $\Omega(\sqrt{NT})$ by  \cite{chen2017note}. More recently, \cite{mnlthomson} develop a Thomson sampling based algorithm for the MNL-bandit problem with worst-case regret of $\Tilde{O}(\sqrt{NT})$ and show that their algorithm performs well in practice compared to previous methods. Both algorithms of \cite{mnlucb,mnlthomson} rely on an epoch-based approach to estimation in the sense that multiple time steps (epoch) are used to get a single sample from each estimator of the unknown parameters.

Our problem is closely related to the multi-armed bandit (MAB) paradigm \citep{robbins1952some}. There is a large body of literature studying the multi-armed bandit problem (and its extensions) in non-stationary environments. Some of these works consider a structured model of non-stationarity including the restless bandits \citep{whittle1988restless,tekin2012online}, the rested bandits \citep{gittins1979bandit,tekin2012online} and more \citep{slivkins2008adapting,rottingbandits,kleinberg2018recharging}. Other works consider an adversarial model of non-stationarity such as ours. Among these, works such as \cite{auer2002nonstochastic,bubeck2012best,seldin2014one,auer2016algorithm} where the performance is compared to the single best action, and others considering a dynamic notion of regret. In these later works, bounds on the worst-case dynamic regret is given in function of parameters characterizing the degree of non-stationarity of the environment. Two widely considered parameters are the number of switches $L$ which is the number of times the unknown parameters change and the variation $\Delta$ which characterizes the total amount of change in the unknown parameters over the whole time horizon \citep{garivier2011upper,besbes2014stochastic,besbes2015non,liu2018change,luo2018efficient,cheung2019learning,chen2019new,auer2019adaptively,chen2021combinatorial,wei2021non}. Other parameters have also been considered  in the literature \citep{wei2016tracking,suk2022tracking,abbasi2022new}.

\section{Problem formulation}

In the MNL-Bandit problem, a decision maker is given a choice space of $N$ items with payoffs (or revenues) $r_1, \dots, r_N$ and an integer $K \leq N$. Nature chooses a number of decision rounds $T$ and chooses an attraction parameters $\omega_j(t)$ for every round $t$ and item $j$. Let $\bm{\omega}(t) := (\omega_1(t), \dots, \omega_N(t))$ denote the vector of attraction parameters at round $t$. In each round $t$, the decision maker selects a subset $S(t) \subset [N]$ of at most $K$ items (called assortment). Then nature chooses (or purchases) an item $j(t) \in S(t)$ (or no item at all) according to an MNL model of choice with attraction parameters $\bm{\omega}(t)$. Finally, the decision maker collects a payoff $r(t) = r_{j(t)}$. 

The MNL choice model with attraction parameters $\bm{\omega}=(\omega_1, \dots, \omega_N)$ is such that for every $S \subset [N]$, when the offered assortment is $S$, nature chooses item $j \in S$ with probability, $p^{\bm{\omega}}(j, S) := \omega_j/(1 + \sum_{r \in S} \omega_r),$ or no item at all with probability $p^{\bm{\omega}}(0, S) := 1/(1 + \sum_{r \in S} \omega_r).$
One can think of the ``no-choice" or ``no-purchase" option as an extra item (item $0$) that is always available to nature and whose attraction parameter is $\omega_0 = 1$. Following \cite{mnlucb}  we make the assumption that for every time step $t$, $\omega_j(t) \leq \omega_0=1$. This assumption is usually verified in practice. In online retailing for example, the most common choice of a customer is to buy nothing at all. We also assume that $r_j \leq 1$ for every item $j$. This assumption is made for simplicity and does not change our results quantitatively as soon as the rewards are bounded by some positive constant. 

The goal of the decision maker is to maximize, at each round $t$, her expected payoff given by,
$
R(S(t), \bm{\omega}(t)) := \sum_{j \in S(t)} r_j p^{\bm{\omega}(t)}(j, S(t)).
$
However, as the attraction parameters $(\bm{\omega}(t))_{t \in [1,T]}$ are unknown, the decision maker seeks instead to minimize her dynamic regret defined as,
$$
\dreg = \sum_{t=1}^T R(S^*(t),\bm{\omega}(t)) - r(t),
$$
where $S^*(t)$ is the assortment maximizing the expected payoff under the parameters $\bm{\omega}(t)$. 

We design an online algorithm for the decision maker and evaluate the worst-case dynamic regret of our algorithm as a function of the time horizon $T$, the number of switches $L$, the variation $\deltaIk$ and the other known parameters of the problem.


\section{Algorithm}
\label{sec:algo}

In this section, we present our algorithm for the MNL-Bandit problem in non-stationary environments that we refer to as Exploration-Exploitation algorithm for MNL-bandit in Non-Stationary Environments (EEMNSE). Our algorithm involves a substantial adaptation of the UCB algorithm of \cite{mnlucb} (which only works in stationary setting) to near-stationary environments. We begin by a formal definition of ``a near-stationary environment". We then present our algorithm and the dynamic regret upper-bound it achieves. Finally, we discuss the keys ideas and techniques of the analysis.

\begin{definition}
\label{def:near-stat}
    Consider a function $\Delta:[1,T] \rightarrow [0,\infty)$ such that for every $t \in [1,T-1]$ it holds that $\Delta(t) \geq \max_{S \subset [N]:|S| \leq K} |R(S, \bm{\omega}(t)) - R(S, \bm{\omega}(t+1))|$. We refer to this function as a non-stationarity measure. Let $\rho:[1,T] \rightarrow [0, \infty)$ be a non-increasing function such that $\rho(t) \geq \frac{1}{\sqrt{t}}$ and the map $t \rightarrow t \cdot \rho(t)$ is non-decreasing. Then the $(\Delta, \rho)$-near-stationary part of the environment is defined as the time rounds $t \in [1,T]$ such that $\sum_{\tau=1}^{t-1} \Delta(\tau) \leq \rho(t)$.

    We say that an algorithm performs well in $(\Delta, \rho)$-near-stationary environments if the algorithm does not require knowledge of $\Delta$ and for every instance of the problem the algorithm outputs a reward upper-bound $\Hat{R}_t \in [0,1]$ at the beginning of each round $t$ such that the following two conditions hold with probability at least $1-O(\frac{1}{T})$ for every round $t$ belonging to the $(\Delta, \rho)$-near-stationary part of the environment:
    \vspace{-5mm}
    \begin{center}
    \begin{subequations}
    \noindent
    \begin{minipage}{0.5\textwidth}
      \begin{equation} \label{eq:cond1}
      \Hat{R}_t \geq \min_{\tau \in [1,t]}R(S^*(\tau), \bm{\omega}(\tau)) - \sum_{\tau=1}^{t-1}\Delta(\tau)
      \end{equation}
    \end{minipage}%
    \begin{minipage}{0.5\textwidth}
      \begin{equation} \label{eq:cond2}
      \frac{1}{t}\sum_{\tau=1}^{t} \Hat{R}_\tau - r(\tau) \leq \rho(t) + \sum_{\tau=1}^{t-1}\Delta(\tau)
      \end{equation}
    \end{minipage}
    \end{subequations}
    \end{center}
\end{definition}

\noindent Intuitively, the $(\Delta, \rho)$-near-stationary part of the environment is the part where the variability of the attraction parameters (measured through $\Delta$) is not very large (compared to $\rho$). Let us now present our algorithm for MNL-Bandit problem:

\paragraph{Exploration-Exploitation algorithm for MNL-bandit in Non-Stationary Environments (EEMNSE).}

\cite{wei2021non} give a meta algorithm (referred to as MASTER) such that given a base algorithm with good performance in $(\Delta, \rho)$-near-stationary environments for some $\Delta$ and $\rho$, MASTER schedules the base algorithm in a multi-scale manner and uses the reward upper-bounds $\hat{R}_t$ to perform at each time step non-stationarity tests to detect a change in the unknown parameters. Once a change is detected, MASTER restarts. We refer the reader to \cite{wei2021non} for a detailed description. \cite{wei2021non}  bound  the worst-case dynamic regret of the combination of MASTER and the base algorithm as a function of $\Delta$, $\rho$ and other parameters of the problem. Our algorithm is given by the combination of MASTER with the base algorithm that is an adaptation of the UCB algorithm of \cite{mnlucb} to near-stationary environments. In particular, different upper-confidence bounds are needed as a result of new concentration bounds required for good performance in near-stationary environments. We present our adaptation in Algorithm~\ref{alg:near-stationary}. Algorithm~\ref{alg:near-stationary} proceeds in epochs. At each epoch, we offer the assortment that maximizes the expected payoff with respect to the current upper confidence bounds (UCBs) of the attraction parameters. When a ``no-purchase" happens, the epoch ends and the UCBs are updated.

\SetKwInput{Init}{Initialization}
\SetKwComment{Comment}{/*}{*/}
\newcommand\mycommfont[1]{\footnotesize\ttfamily\textcolor{blue}{#1}}
\SetCommentSty{mycommfont}
\begin{algorithm}
\caption{Exploration-Exploitation for MNL-Bandit in near-stationary environments}\label{alg:near-stationary}
\Init{For each item $j \in [N],$ let $n_j = 0$, $\tilde{\omega}_j = 0$, $\Bar{\omega}_j = 0$, $\hat{\omega}_j = 192\log NT$.}

\For{$t = 1, \dots, T$}{
    
    Offer the assortment $S = \argmax_{C \subset [N]: |C| \leq K} R(C, \hat{\bm{\omega}})$; item $j(t)$ is purchased.
    
    \eIf(\hfill\quad\quad\quad\quad\quad\quad\quad\quad\quad\tcp*[h]{a no-purchase happened and epoch ends here}){$j(t) = 0$}{
    \For{$j \in S$}{
        $n_j \gets n_j+1$ \tcp*[r]{Update the number of epochs where $j$ was proposed}
        
        $\Bar{\omega}_j \gets \frac{n_i-1}{n_i} \Bar{\omega}_j + \frac{1}{n_i} \Tilde{\omega}_j$ \tcp*[r]{Update the mean of the number of purchases}

        $\hat{\omega}_j \gets \Bar{\omega}_j + \sqrt{\frac{192\Bar{\omega}_j \log NT}{n_j}} + \frac{192 \log NT}{n_j}$ for all $j \in [N]$ \tcp*[r]{Update the UCBs}

        $\Tilde{\omega}_j \gets 0$ \tcp*[r]{Reset the number of purchases within an epoch}
    }
    }{
        $\Tilde{\omega}_{j(t)} \gets \Tilde{\omega}_{j(t)} + 1$ \tcp*[r]{Update the number of purchases within current epoch}
    }
}
\end{algorithm}

\paragraph{Regret bounds.}
In order to bound the dynamic regret of EEMNSE, we evaluate the regret accumulated by the base algorithm (Algorithm~\ref{alg:near-stationary}) when the environment is near-stationary. In particular, we show the existence of functions $\Delta$ and $\rho$ such that Algorithm \ref{alg:near-stationary} performs well in $(\Delta, \rho)$-near-stationary environments and that EEMNSE achieves an optimal dynamic regret for the MNL-Bandit problem (up to logarithmic factors). More precisely, we show the following theorem:

\begin{restatable}{theorem}{restateregret}
\label{thm:DReg}
For the choices, $$\Delta(t) = 26 \lVert \bm{\omega}(t) - \bm{\omega}(t+1) \rVert_{L^{2K}_\infty}, \quad 
\forall t \in [1,T-1],$$ and $$\rho(t) = (149\log{NT})^{\frac{3}{2}}\sqrt{\frac{N}{t}} + (55\log{NT})^{3} \frac{N}{t} + \sqrt{\frac{2\log T}{t}}, \quad \forall t \in [1,T],$$ Algorithm~\ref{alg:near-stationary} performs well in $(\Delta, \rho)$-near-stationary environments with $\hat{R}_{t} = R(S(t), \hat{\bm{\omega}}(t))$, where $\hat{\bm{\omega}}(t)$ is the vector of the upper confidence bounds at round $t$. By Theorem 2 of \cite{wei2021non}, this further implies that, without knowledge of $\deltaIk$ and $L$, EEMNSE guarantees with high probability a dynamic regret, $$\dreg = \Tilde{O}\left( \min \left\{ \sqrt{NTL}\;,\; N^{\frac{1}{3}}(\deltaIk)^{\frac{1}{3}} T^{\frac{2}{3}} + \sqrt{NT}\right\}\right).$$
\end{restatable}

\paragraph{Notation.}

We adopt the unified notation $X^{(t)}_j(i)$ to denote the value of a variable $X$ (e.g., the number of purchases) at the $i$-th step (a step can be a time step, an epoch or just an index, if $X$ does not depend on any particular notion of step we drop the index $i$), for item $j \in [N]$ (again if $X$ does not depend on items we drop the subscript), and between the time rounds $1$ and $t$ (when the time frame matters otherwise we drop the superscript). For every $i, k \geq 1$ such that $i \leq k$, we use $X^{(t)}_j(i:k)$ to denote the random process (or sequence of random variables) $X^{(t)}_j(i), \dots, X^{(t)}_j(k)$. 

With this convention, we define the following random variables and processes related to Algorithm~\ref{alg:near-stationary}: we define $e^{(t)}_j$ to be the number of epochs where $j$ was proposed in the assortment up to time $t$ (included). Note that this also includes an incomplete epoch that did not end before $t$. We define $l^{(t)}_j\left(1:e^{(t)}_j\right)$ to be the epochs where $j$ was proposed up to time $t$ (note that $l^{(t)}_j\left(e^{(t)}_j\right)$ might be an incomplete epoch). Define $\Tilde{\omega}^{(t)}_j\left(1:e^{(t)}_j\right)$ to be the number of purchases of item $j$ in the epochs $l^{(t)}_j\left(1:e^{(t)}_j\right)$ respectively. We refer to $\Tilde{\omega}^{(t)}_j\left(1:e^{(t)}_j\right)$ as the purchases process. We also define $\bm{\hat{\omega}}(t) = (\hat{\omega}_1(t), \dots, \hat{\omega}_N(t))$ and $\bm{\bar{\omega}}(t) = (\bar{\omega}_1(t), \dots, \bar{\omega}_N(t))$ to be the vector of upper confidence bounds $(\hat{\omega}_1, \dots, \hat{\omega}_N)$ and averages $(\bar{\omega}_1, \dots, \bar{\omega}_N)$ at round $t$ respectively, and $n_j(t)$ to be the number of complete epochs before $t-1$ where $j$ was proposed (only epochs that ended before or at $t-1$ are counted). Let $\Tilde{n}_j(t) = \max\{n_j(t), 1\}$.

For every $t \in [1,T]$ and item $j \in [N]$, define $\delta^{(t)} := \sum_{\tau=1}^{t-1} \lVert \bm{\omega}(\tau) - \bm{\omega}(\tau+1) \rVert_{L^{2K}_\infty}$ and $\delta^{(t)}_j := \sum_{\tau=1}^{t-1} |\omega_j(\tau) - \omega_j(\tau+1)|$. Finally, for every $\tau \in [1,t]$, let $\mu^{+(t)}_j(\tau) = \left(\omega_j(\tau) +\delta^{(t)}_j\right)\frac{1+ \delta^{(t)}}{1 - \delta^{(t)}}$ and $\mu^{-(t)}_j(\tau) = \left(\omega_j(\tau) -\delta^{(t)}_j\right)^+\frac{1- \delta^{(t)}}{1 + \delta^{(t)}}$, where $(a)^+$ denotes the positive part of $a \in \mathbb{R}$. Note that for integers $m \leq n$ we use $[m,n]$ to denote the set $\{m, \dots, n\}$ and $[m]$ to denote the set $[1,m]$.

The next section is dedicated
to the proof of Theorem \ref{thm:DReg}. But before moving to the proof, we discuss briefly the challenges presented by the analysis and the key ideas we use to overcome them.

\paragraph{Key ideas.} 

Consider $\Delta$ and $\rho$ given in Theorem \ref{thm:DReg}. We show that conditions \eqref{eq:cond1} and \eqref{eq:cond2} hold with $\hat{R}_t = R(S(t), \hat{\bm{\omega}}(t))$. Fix a time step $t$ in the $(\Delta,\rho)$-near-stationary part of the environment. Conditions \eqref{eq:cond1} and \eqref{eq:cond2} require bounding the change in the expected payoff function when using the upper confidence bounds instead of the true attraction parameters within the time frame $[1,t]$. We do this in two steps: in the first step, we evaluate the gap between the upper confidence bounds $\bm{\hat{\omega}}(t)$ and the true attraction parameters $\bm{\omega}(\tau)$ for any $\tau \in [1,t]$, then in the second step we use a special Lipschitz property of the expected payoff function to translate this gap in terms of expected payoff. In the first step, the starting point is to derive concentration bounds for the average $\Bar{\omega}_j(t) = \frac{1}{\tilde{n}_j(t)}\sum_{l=1}^{n_j(t)} \Tilde{\omega}^{(t)}_j(l)$ for every $j\in [N]$. Recall that $\Tilde{\omega}^{(t)}_j(l)$, the number of purchases of $j$ in epoch $l$, is used as an estimator of the true attraction parameter $\omega_j$. Ideally we would like this estimator to be a sub-Gaussian centered around $\omega_j$. In a stationary environment, \cite{mnlucb} show that this is indeed the case and that the distribution of $\Tilde{\omega}^{(t)}_j(l)$ (conditioned on the assortment offered in the epoch) follows a geometric distribution centered around $\omega_j$ 
(as we mention earlier, the distribution of $\Tilde{\omega}^{(t)}_j(l)$ is  not exactly geometric and centered around $\omega_j$ due to the truncation from a finite time horizon).
In a non-stationary environment however, the attraction parameters change during an epoch and the resulting estimators are now biased and their distributions depend on the adversarial changes in the attraction parameters. Moreover, the distributions are not necessarily independent across epochs. This violates crucial properties of the estimators used in the analysis of the stationary setting.  

To overcome these issues, we introduce the uninterrupted purchases process (UPP). This process simulates the number of purchases of item $j$ within each epoch it was offered in, if the algorithm was to continue in a stationary environment after $t$ where $\bm{\omega}(t') = \bm{\omega}(t)$ for $t' \geq t$ and offer the same assortment $S=[N]$ in the epochs starting after $t$. (UPP) has the advantage of having ``nicer" distributions than the purchases process and, as we will see, deriving concentration bounds for (UPP) is sufficient to get concentration bounds for the mean number of purchases. In a stationary environment, (UPP) consists of a sequence of geometric random variables that allows us to solve the truncation issue. In a non-stationary environment, we sandwich the variables of (UPP) between two geometric random variables with close mean in the usual stochastic order, allowing us to control the bias caused by non-stationarity by two ``nice" distributions. More specifically, we sandwich the variables of (UPP) between two geometric random variables of mean $\mu^{+(t)}_j(\tau)$ and $\mu^{-(t)}_j(\tau)$. Our stochastic bounds are tight in the sense that they lead to optimal regret upper-bounds. To get our stochastic bounds, we describe (UPP) in terms of i.i.d. Gumbel(0,1) variables, then we give a similar description of geometric random variables using Gumbel(0,1) variables and prove our stochastic bounds using a coupling argument. The next step is to derive the concentration bounds. We derive non-symmetric concentration bounds in the sense that one of our bounds is with respect to $\mu^{+(t)}_j(\tau)$ and the other is with respect to $\mu^{-(t)}_j(\tau)$. We show that this is enough for our purposes. Finally, we translate the gap between the upper confidence bounds and the true parameters into a gap in the expected revenue function using a special Lipshitz property of the revenue function. We present the detailed proofs in the Appendix.

\section{Analysis}

In the sequel, we fix $\Delta$ and $\rho$ as defined in Theorem \ref{thm:DReg} and let $\hat{R}_{t} = R(S(t), \hat{\bm{\omega}}(t))$ for every $t \in [1,T]$. Our goal is to show that $\Delta$ is indeed a non-stationarity measure and that for every $t$ belonging to the $(\Delta, \rho)$-near-stationary part of the environment, conditions \eqref{eq:cond1} and \eqref{eq:cond2} hold with probability at least $1-O(\frac{1}{T})$. Note that for every time step $t \in [1,T]$ such that $\delta^{(t)} \geq \frac{1}{2}$ we have that $\sum_{\tau=1}^{t-1} \Delta(\tau)  = 26 \delta^{(t)} \geq 13 \geq 1$ and the conditions \eqref{eq:cond1} and \eqref{eq:cond2} hold for such $t$. Hence, we only focus on time steps $t$ such that $\delta^{(t)} \leq \frac{1}{2}$.

\subsection{The uninterrupted purchases process (UPP)}

Let $t \in [1,T]$ such that $\delta^{(t)} \leq \frac{1}{2}$ and let $j \in [N]$. We define the uninterrupted purchases process $\zeta^{(t)}_j(1:\infty)$  that simulates the number of purchases of $j$ in the epochs it was proposed in if after time $t$: (i) the algorithm continues in a stationary environment with $\bm{\omega}(t')=\bm{\omega}(t)$ for every $t' \geq t$ and (ii) the algorithm offers the same assortment $S=[N]$ in the epochs starting after $t$. 

Formally, consider the infinite process $X^{(t)}(t+1: \infty)$ simulating the time rounds after $t+1$ and defined recursively as follows: For every $t' \geq t+1$, if none of the variables $j(t)$ or $X^{(t)}(t+1: t'-1)$ is $0$, then $X^{(t)}(t')$ is sampled from a categorical distribution supported in $S(t) \cup \{0\}$, and independent of any other randomness of the problem given
$S(t)$, such that each item $r \in S(t)$ is sampled with probability $\omega_{r}(t) / (1+\sum_{k \in S(t)}\omega_{k}(t))$. Otherwise, $X^{(t)}(t')$ is sampled from a categorical distribution supported in $[N] \cup \{0\}$, and independent of any other randomness of the problem, such that each item $r \in [N]$ is sampled with probability $\omega_{r}(t)/(1+\sum_{k \in [N]}\omega_{k}(t))$. This process simulates imaginary rounds beyond $t$. It first completes the ongoing epoch at round $t$ (offering assortment $S(t)$) then switches to the assortment $[N]$ after this epoch ends. Let $i_k$ denote the $k$-th step $i \geq t+1$ where $X^{(t)}(i) = 0$ happens (by convention $i_0 = t+1$). (UPP) is defined as follows:




\begin{definition}
For every $t \in [1,T]$ such that $\delta^{(t)} \leq \frac{1}{2}$ and $j \in [N]$. The uninterrupted purchases process $\zeta^{(t)}_j(1:\infty)$ is defined as follows: for every epoch $i \leq e^{(t)}_j-1$, $\zeta^{(t)}_j(i) = \Tilde{\omega}^{(t)}_j(i)$. For $i \geq e^{(t)}_j$, we distinguish two cases. The first case is when a no-purchase happened at round $t$, in which case we let $\zeta^{(t)}_j(e^{(t)}_j) = \Tilde{\omega}^{(t)}_j(e^{(t)}_j)$ and for $k \geq e^{(t)}_j+1$, $\zeta^{(t)}_j(k)$ is equal to the number of steps $i \in \left[i_{k-e^{(t)}_j-1}, i_{k-e^{(t)}_j}\right]$ where $X^{(t)}(i)=j$. The second case is when a purchase happened at $t$ and in this case we let $\zeta^{(t)}_j(e^{(t)}_j)$ be the sum of $\Tilde{\omega}^{(t)}_j(e^{(t)}_j)$ and the number of steps $i \in \left[i_{0}, i_{1}\right]$ where $X^{(t)}(i)=j$, and  for $k \geq e^{(t)}_j+1$, $\zeta^{(t)}_j(k)$ is equal to the number of steps $i \in \left[i_{k-e^{(t)}_j}, i_{k-e^{(t)}_j+1}\right]$ where $X^{(t)}(i)=j$.
\end{definition}
We extend the definitions of the starting time of an epoch, the offered assortment $S(\tau)$ and the purchased element $j(\tau)$ at round $\tau$ to the newly defined imaginary epochs of (UPP) as follows,

\begin{definition}
    Let $t \in [1,T]$ such that $\delta^{(t)} \leq \frac{1}{2}$ and let $j \in [N]$. Let $k \geq 1$. The starting time of the epoch (real or imaginary) that defines $\zeta^{(t)}_j(k)$, denoted by $s^{(t)}_j(k)$,  is defined as follows: If $k \leq e^{(t)}_j$, then $s^{(t)}_j(k)$ is the starting time of epoch $l_j^{(t)}(k)$. If a no-purchase happened at $t$, then $s^{(t)}_j(k) = i_{k-e^{(t)}_j-1}$ for every $k \geq e^{(t)}_j + 1$. Otherwise $s^{(t)}_j(k) = i_{k- e^{(t)}_j}$ for every $k \geq e^{(t)}_j + 1$. The offered assortment in the imaginary rounds $\tau \geq t+1$ is defined as follows: for $t \leq \tau < s^{(t)}_j\left(e^{(t)}_j+1\right)$ let $S(\tau)=S(t)$, for $\tau \geq s^{(t)}_j\left(e^{(t)}_j+1\right)$ let $S(\tau) = [N]$. Finally, the purchased element at each round $\tau \geq t+1$ is defined as $j(\tau) = X^{(t)}(\tau)$.
\end{definition}

\begin{remark}
Let $t \in [1,T]$ such that $\delta^{(t)} \leq \frac{1}{2}$ and let $j \in [N]$. Let $k \geq 1$. The random variable $\zeta^{(t)}_j(k)$ is well defined almost surely. In fact, the only case when $\zeta^{(t)}_j(k)$ is not well defined is when there exists an epoch $k'<k$ that lasts an infinite number of time steps. Since the probability of a no-purchase is at least $\frac{1}{1+N}$ at every round, the probability of an epoch $k'$ lasting forever is $0$. A union bound over all the possible $k'$ shows that $\zeta^{(t)}_j(k)$ is well defined almost surely. Another union bound over all $t,j$ and $k$ shows that every $\zeta^{(t)}_j(k)$ is well defined almost surely. We shall restrict ourselves in the sequel and without loss of generality to the sample paths of the algorithm where this holds.
\end{remark}

\paragraph{Sandwiching (UPP) in the usual stochastic order:}

For every time step $s \geq 1$, let $\Gamma(s)$ denote the subsets of $[N]$ such that $S \in \Gamma(s)$ if and only if there exists at least one sample path of the algorithm where $S(s) = S$. Note that in particular, for every $s \leq t$ , every $S \in \Gamma(s)$ is such that $|S| \leq K$. The following lemma sandwiches the variables of (UPP) between two geometric random variables in the stochastic order:


\begin{restatable}{lemma}
{restatesandwich}
\label{lem:sandwich}
Let $t \in [1,T]$ such that $\delta^{(t)} \leq \frac{1}{2}$ and let $j \in [N]$. Let $k \geq 1$, $\tau \in [1,t]$, $s \geq 1$, and $S \in \Gamma(s)$. Let $X$ denote the random variable $\zeta^{(t)}_j(k)$ conditioned on $s^{(t)}_j(k) = s$ and $S(s)=S$. Let $X^{+}$ be a geometric random variable with mean $\mu^{+(t)}_j(\tau)$, and $X^{-}$ be a geometric random variable with mean $\mu^{-(t)}_j(\tau)$. Then,
$$X^{-} \leq_{st} X \leq_{st} X^{+},$$ where $\leq_{st}$ denotes the usual stochastic order.
\end{restatable}

\subsection{Concentration bounds}
Let $t \in [1,T]$ such that $\delta^{(t)} \leq \frac{1}{2}$ and let $j \in [N]$. By leveraging Lemma \ref{lem:sandwich}, we can derive concentration bounds for the random variables $\zeta^{(t)}_j(k)$. We begin by bounding the moment generating function conditioned on $s^{(t)}_j(k) = s$ and $S(s)=S$ for $s \geq 1$ and $S \in \Gamma(s)$.

\begin{restatable}{lemma}{restatemgf}
    Let $t \in [1,T]$ such that $\delta^{(t)} \leq \frac{1}{2}$ and let $j \in [N]$. Let $k \geq 1$, $\tau \in [1,t]$, $s \geq 1$, and $S \in \Gamma(s)$, then
    for every $0 \leq \lambda < \log \left( 1 + \frac{1}{\mu^{+(t)}_j(\tau)} \right)$, we have, $$\mathbb{E}\left(\left. e^{\lambda \zeta^{(t)}_j(k)} \right| s^{(t)}_j(k) = s \;,\; S(s)=S\right) \leq \frac{1}{1-\mu^{+(t)}_j(\tau) \cdot (e^\lambda - 1)}.$$ And for every $\lambda \leq 0$, we have, $$\mathbb{E}\left(\left. e^{\lambda \zeta^{(t)}_j(k)} \right| s^{(t)}_j(k) = s \;,\; S(s)=S\right) \leq \frac{1}{1-\mu^{-(t)}_j(\tau) \cdot (e^\lambda - 1)}$$
\end{restatable}

We are now ready to derive concentration bounds for the variables $\zeta^{(t)}_j(k)$. In particular, we have the following lemma,

\begin{restatable}{lemma}{restateconcentration}
\label{lem:concentrationZeta}
Let $t \in [1,T]$ such that $\delta^{(t)} \leq \frac{1}{2}$ and let $j \in [N]$. Let $k \geq 1$ and $\tau \in [1,t]$. For every $\eta > 0$, we have,
\begin{align*}
    \mathbb{P}\left(\frac{1}{k}\sum_{i=1}^{k} \zeta^{(t)}_j(i) < (1-\eta)\mu^{-(t)}_j(\tau)\right) 
    \leq 
    e^{- \frac{k\eta^2}{24}\mu^{-(t)}_j(\tau)} 
\end{align*}
and,
\begin{align*}
    \mathbb{P}\left(\frac{1}{k}\sum_{i=1}^{k} \zeta^{(t)}_j(i) > (1+\eta)\mu^{+(t)}_j(\tau) \right) \leq
    e^{-\frac{k \min \{\eta, \eta^2\}}{196}\mu^{+(t)}_j(\tau)}.
\end{align*}
\end{restatable}

Recall that $\Bar{\omega}_j(t) := \frac{1}{\tilde{n}_j(t)} \sum_{i=1}^{n_j(t)} \Tilde{\omega}^{(t)}_j(i)$ denotes the mean of the number of purchases of $j$ taken over the $n_j(t)$ complete epochs in $[1,t-1]$ where $j$ was proposed. Since $n_j(t)$ includes only the epochs that ended before or at $t-1$ we have $\Tilde{\omega}^{(t)}_j(i) = \zeta^{(t)}_j(i)$ for every $i \leq n_j(t)$, implying that $\Bar{\omega}_j(t) := \frac{1}{\tilde{n}_j(t)} \sum_{i=1}^{n_j(t)} \zeta^{(t)}_j(i)$. By leveraging this fact and Lemma \ref{lem:concentrationZeta} above, we prove the following concentration bounds for $\Bar{\omega}_j(t)$:


\begin{restatable}{lemma}{restateconcentrationUCB} 
\label{lem:concentrationUCB}
Let $t \in [1,T]$ such that $\delta^{(t)} \leq \frac{1}{2}$ and let $j \in [N]$. Let $\tau \in [1,t]$. The following concentration bounds hold,
    \begin{align*}
    \mathbb{P}\left(\Hat{\omega}_j(t) < \mu^{-(t)}_j(\tau)\right) \leq \frac{2}{NT^3}
    \end{align*}
and
    \begin{align*}
    \mathbb{P}\left(\Hat{\omega}_j(t) > \mu^{+(t)}_j(\tau) + \sqrt{\frac{2266 \mu^{+(t)}_j(\tau) \log NT}{\tilde{n}_j(t)}} + \frac{1364\log NT}{\tilde{n}_j(t)}\right) \leq \frac{2}{NT^3}.
    \end{align*}
\end{restatable}

\subsection{Bounding the length of the epochs}
\label{subsec:epochlength}

The last ingredient we need in our analysis is to show that the length of an epoch is concentrated around its mean value. In the analysis of \cite{mnlucb} for stationary environments, the length of the epochs was replaced by its mean value using the law of conditional expectation (the transition from (A.13) to (A.14) in \cite{mnlucb}). However, this does not hold as the total number of epochs $L$ is not $\mathcal{H}_{l-1}$-measurable (where $\mathcal{H}_{l-1}$ is the information available upto epoch $l-1$). We adopt a different approach and bound the length of the (real) epochs with high probability. By doing so, we only loose an additional $O(\log NT)$ in the regret. More precisely, we prove the following lemma:

\begin{restatable}{lemma}{restatelengthepoch}
\label{lem:lengthepochs}
Let $t \in [1,T]$ such that $\delta^{(t)} \leq \frac{1}{2}$ and let $j \in [N]$. Let $u \in [1,t]$. We have,
\begin{align*}
    \mathbb{P}\left(\bigcup_{\substack{1 \leq i \leq e^{(t)}_j}} \left\{\left| l_j^{(t)}(i) \right| > 5 \log NT \left(1 + \sum_{k \in S\left(s_j^{(t)}(i)\right)} \omega_k(u) + \delta^{(t)}_k\right) \right\} \right) \leq \frac{1}{NT^3}
\end{align*}
\end{restatable}

\subsection{Putting all together} The above elements can be combined to prove Theorem \ref{thm:DReg}. In the proof, we begin by showing that $\Delta$ is indeed a non-stationarity measure, then we prove that the conditions \eqref{eq:cond1} and \eqref{eq:cond2} hold with probability at least $1-O(\frac{1}{T})$ for every $t \in [1,T]$ such that $\sum_{\tau=1}^{t-1} \Delta(\tau) \leq \rho(t)$ and $\delta^{(t)} \leq \frac{1}{2}$ with $\hat{R}_\tau = R(S(\tau), \bm{\hat{\omega}}(\tau))$. The later fact leverages the previous lemmas along with a special Lipschitz property of the expected payoff function (Lemma A.3 from \cite{mnlucb}). The proof is given in Appendix \ref{apx:analysis}.

\section{Lower bounds}
In this section, we give lower bounds on the regret achievable by any polynomial time algorithm for MNL-Bandit in a non-stationary environment as a function of $L$ and $\deltaIk$. We also discuss our choice of the variation $\deltaIk$ instead of $\deltaI$.

\paragraph{Dependence on $L$.} The following theorem shows the regret upper-bound of $\Tilde{\mathcal{O}}(\sqrt{NTL})$ achieved by our algorithm is optimal (up to logarithmic factors):

\begin{restatable}{theorem}{restatedependenceL}
\label{thm:dependenceL}
    Fix  $T, L, N$ and $K$ and suppose that $K \leq \frac{N}{4}$. Then, for every polynomial time algorithm $\mathcal{A}$ for the MNL-Bandit problem in non-stationary environments, there exists an instance with parameters $N$ and $K$ and at most $L$ switches such that $\mathcal{A}$ accumulates a regret of at least $C\cdot \min(\sqrt{NLT}, T)$ over $T$ rounds, where $C > 0$ is an absolute constant independent of $T, L, N$ and $K$.
\end{restatable}

\noindent Our proof of Theorem \ref{thm:dependenceL} uses the results of \cite{chen2017note} for stationary environments. At a high level, given an algorithm $\mathcal{A}$, we subdivide the time horizon into $\left\lceil\frac{T}{L}\right\rceil$ windows, then we construct our adversarial instance against $\mathcal{A}$ recursively as follows: Suppose we fixed our adversarial instance in the first $w-1$ windows. For window $w$ we use the stationary adversarial instance given by \cite{chen2017note} against $\mathcal{A}$ (given the previously fixed windows $w' \leq w-1$). We get a non stationary instance with at most $L$ switches and show that $\mathcal{A}$ accumulates an expected regret of at least $C\cdot \min(\sqrt{NLT}, T)$ against this instance. The proof of Theorem \ref{thm:dependenceL} is given in Appendix \ref{apx:lowerbounds}.

\paragraph{Dependence on $\deltaIk$.} The following theorem shows the regret upper-bound of $\Tilde{\mathcal{O}}(N^{\frac{1}{3}}(\deltaIk)^{\frac{1}{3}}T^\frac{2}{3})$ achieved by our algorithm is optimal (up to logarithmic factors):

\begin{restatable}{theorem}{restatedependenceDelta}
\label{thm:dependenceDelta}
    Fix  $T, \Delta, N$ and $K$ and suppose that $K \leq \frac{N}{4}$ and that $\Delta \in [\frac{1}{N}, \frac{T}{N}]$. Then, for every polynomial time algorithm $\mathcal{A}$ for the non-stationary MNL-Bandit problem, there exists an instance with parameters $N$ and $K$ and variation $\deltaIk \leq \Delta$ such that $\mathcal{A}$ accumulates a regret of at least $C\cdot N^{\frac{1}{3}}\Delta^{\frac{1}{3}}T^\frac{2}{3}$ over a time horizon $T$, where $C>0$ is an absolute constant independent of $T, \Delta, N$ and $K$. Moreover, this instance is such that $\deltaIk \geq \frac{K}{2} \cdot \deltaI$, where  $\deltaI$ denotes the variation in norm $L_{\infty}$ of the attraction parameters of the instance.
\end{restatable}

\noindent To prove Theorem \ref{thm:dependenceDelta}, we divide the time horizon into windows of equal length $M$ and construct an adversarial instance against an algorithm $\mathcal{A}$ recursively: As before, given that we fixed our instance in the first $w-1$ windows, in window $w$, we use a stationary instance such as in \cite{chen2017note} against $\mathcal{A}$ (given the previously fixed windows $w' \leq w-1$). With a good choice of $M$, our instance has a variation at most $\Delta$ and is such that $\mathcal{A}$ accumulates a regret of at least $C \cdot N^{\frac{1}{3}} \Delta^\frac{1}{3}T^{\frac{2}{3}}$. We also choose slightly different instances for the even and odd windows so that $\deltaIk \geq \frac{K}{2} \cdot \Delta_{\infty}$. The proof is given in Appendix \ref{apx:lowerbounds}.

\paragraph{The choice of the variation $\deltaIk$.}
In other related bandit settings, including the classical multi-armed bandit MAB (\cite{besbes2014stochastic,wei2021non}) and the combinatorial semi-bandit \cite{chen2021combinatorial}, the dynamic regret bounds are given as a function of the variation in norm $L_{\infty}$. For example, in MAB, a parameter-free upper-bound of $\Tilde{O}(N^{\frac{1}{3}}(\deltaI)^{\frac{1}{3}}T^\frac{2}{3})$ can be achieved in non-stationary environments where $N$ is the number of the arms and $\deltaI$ is the variation in norm $L_{\infty}$ of the mean rewards of the arms. This poses the question of whether one can achieve a similar upper-bound in our setting (note that the lower-bound $\Tilde{\Omega}(N^{\frac{1}{3}}(\deltaIk)^{\frac{1}{3}}T^\frac{2}{3})$ does not exclude this possibility as the worst-case instances might be such that the variation in both norms $L_{\infty}^{2K}$ and $L_{\infty}$ coincide). However, in our setting, the adversarial instances we construct in Theorem \ref{thm:dependenceDelta} are such that $\deltaIk \geq \frac{K}{2} \cdot \deltaI$. This implies that, in terms of $\deltaI$, no polynomial time algorithm can achieve a regret better than $\Tilde{\Omega} ((NK)^{\frac{1}{3}}\Delta_{\infty}^{\frac{1}{3}}T^\frac{2}{3})$. Since $\deltaIk \leq 2K \cdot \deltaI$, the upper-bound $\Tilde{O} (N^{\frac{1}{3}}(\deltaIk)^{\frac{1}{3}}T^\frac{2}{3})$ is a strictly better characterization of the regret than $\Tilde{O} ((KN)^{\frac{1}{3}}(\deltaI)^{\frac{1}{3}}T^\frac{2}{3})$ justifying our choice of the variation $\deltaIk$ as a more convenient variation for the MNL-Bandit setting.

\section*{Conclusion and further directions}

Motivated by realistic settings, we study the MNL-Bandit problem in non-stationary environments. We design an algorithm with optimal dynamic regret bounds (up to logarithmic factors). Our analysis requires new ideas and techniques including an uninterrupted purchases process, stochastic bounds for the bias, new concentration bounds etc. Our work leaves a number of interesting future directions: In many real-word settings, the non-stationarity of the attraction parameters follow structured patterns (some products are more ``attractive" when they are first introduced into the market, other products exhibit seasonal attraction patterns etc.), this leaves the question of whether such structures can be leveraged to improve the learning algorithms for MNL-Bandit. Another interesting question is the
empirical evaluation of our algorithm on real (or artificial) data.



\printbibliography

@article{Rusmevichienton,
  title={Dynamic assortment optimization with a multinomial logit choice model and capacity constraint},
  author={Rusmevichientong, Paat and Shen, Zuo-Jun Max and Shmoys, David B},
  journal={Operations research},
  year={2010}
}

@article{saurezeevi,
  title={Optimal dynamic assortment planning with demand learning},
  author={Saur{\'e}, Denis and Zeevi, Assaf},
  journal={Manufacturing \& Service Operations Management},
  year={2013}
}

@inproceedings{mnlucb,
  title={A near-optimal exploration-exploitation approach for assortment selection},
  author={Agrawal, Shipra and Avadhanula, Vashist and Goyal, Vineet and Zeevi, Assaf},
  booktitle={Proceedings of the 2016 ACM Conference on Economics and Computation},
  year={2016}
}

@inproceedings{mnlthomson,
  title={Thompson sampling for the mnl-bandit},
  author={Agrawal, Shipra and Avadhanula, Vashist and Goyal, Vineet and Zeevi, Assaf},
  booktitle={Conference On Learning Theory (COLT)},
  year={2017}
}

@article{chen2017note,
  title={A note on a tight lower bound for mnl-bandit assortment selection models},
  author={Chen, Xi and Wang, Yining},
  journal={arXiv preprint arXiv:1709.06109},
  year={2017}
}

@article{gittins1979bandit,
  title={Bandit processes and dynamic allocation indices},
  author={Gittins, John C},
  journal={Journal of the Royal Statistical Society: Series B (Methodological)},
  year={1979}
}

@article{whittle1988restless,
  title={Restless bandits: Activity allocation in a changing world},
  author={Whittle, Peter},
  journal={Journal of applied probability},
  year={1988}
}

@article{tekin2012online,
  title={Online learning of rested and restless bandits},
  author={Tekin, Cem and Liu, Mingyan},
  journal={IEEE Transactions on Information Theory},
  year={2012}
}

@inproceedings{rottingbandits,
 author = {Levine, Nir and Crammer, Koby and Mannor, Shie},
 title = {Rotting Bandits},
 booktitle = {Advances in Neural Information Processing Systems},
 year = {2017}
}

@inproceedings{kleinberg2018recharging,
  title={Recharging bandits},
  author={Kleinberg, Robert and Immorlica, Nicole},
  booktitle={2018 IEEE 59th Annual Symposium on Foundations of Computer Science (FOCS)},
  year={2018}
}

@inproceedings{slivkins2008adapting,
  title={Adapting to a Changing Environment: the Brownian Restless Bandits.},
  author={Slivkins, Aleksandrs and Upfal, Eli},
  booktitle={Conference On Learning Theory (COLT)},
  year={2008}
}

@inproceedings{garivier2011upper,
  title={On upper-confidence bound policies for switching bandit problems},
  author={Garivier, Aur{\'e}lien and Moulines, Eric},
  booktitle={Algorithmic Learning Theory: 22nd International Conference, ALT 2011, Espoo, Finland, October 5-7, 2011. Proceedings 22},
  year={2011},
}

@article{wei2016tracking,
  title={Tracking the best expert in non-stationary stochastic environments},
  author={Wei, Chen-Yu and Hong, Yi-Te and Lu, Chi-Jen},
  journal={Advances in neural information processing systems},
  year={2016}
}

@inproceedings{liu2018change,
  title={A change-detection based framework for piecewise-stationary multi-armed bandit problem},
  author={Liu, Fang and Lee, Joohyun and Shroff, Ness},
  booktitle={Proceedings of the AAAI Conference on Artificial Intelligence},
  year={2018}
}

@inproceedings{luo2018efficient,
  title={Efficient contextual bandits in non-stationary worlds},
  author={Luo, Haipeng and Wei, Chen-Yu and Agarwal, Alekh and Langford, John},
  booktitle={Conference On Learning Theory (COLT)},
  year={2018}
}

@inproceedings{auer2019adaptively,
  title={Adaptively tracking the best bandit arm with an unknown number of distribution changes},
  author={Auer, Peter and Gajane, Pratik and Ortner, Ronald},
  booktitle={Conference On Learning Theory (COLT)},
  year={2019}
}

@article{besbes2014stochastic,
  title={Stochastic multi-armed-bandit problem with non-stationary rewards},
  author={Besbes, Omar and Gur, Yonatan and Zeevi, Assaf},
  journal={Advances in neural information processing systems},
  year={2014}
}

@article{besbes2015non,
  title={Non-stationary stochastic optimization},
  author={Besbes, Omar and Gur, Yonatan and Zeevi, Assaf},
  journal={Operations research},
  year={2015}
}

@inproceedings{yang2016tracking,
  title={Tracking slowly moving clairvoyant: Optimal dynamic regret of online learning with true and noisy gradient},
  author={Yang, Tianbao and Zhang, Lijun and Jin, Rong and Yi, Jinfeng},
  booktitle={International Conference on Machine Learning},
  year={2016}
}

@inproceedings{hazan2009efficient,
  title={Efficient learning algorithms for changing environments},
  author={Hazan, Elad and Seshadhri, Comandur},
  booktitle={Proceedings of the 26th annual international conference on machine learning},
  year={2009}
}

@article{auer2002nonstochastic,
  title={The nonstochastic multiarmed bandit problem},
  author={Auer, Peter and Cesa-Bianchi, Nicolo and Freund, Yoav and Schapire, Robert E},
  journal={SIAM journal on computing},
  year={2002}
}

@inproceedings{chen2019new,
  title={A new algorithm for non-stationary contextual bandits: Efficient, optimal and parameter-free},
  author={Chen, Yifang and Lee, Chung-Wei and Luo, Haipeng and Wei, Chen-Yu},
  booktitle={Conference On Learning Theory (COLT)},
  year={2019}
}

@inproceedings{bubeck2012best,
  title={The best of both worlds: Stochastic and adversarial bandits},
  author={Bubeck, S{\'e}bastien and Slivkins, Aleksandrs},
  booktitle={Conference On Learning Theory (COLT)},
  year={2012}
}

@inproceedings{seldin2014one,
  title={One practical algorithm for both stochastic and adversarial bandits},
  author={Seldin, Yevgeny and Slivkins, Aleksandrs},
  booktitle={International Conference on Machine Learning},
  year={2014}
}

@inproceedings{auer2016algorithm,
  title={An algorithm with nearly optimal pseudo-regret for both stochastic and adversarial bandits},
  author={Auer, Peter and Chiang, Chao-Kai},
  booktitle={Conference On Learning Theory (COLT)},
  year={2016}
}

@inproceedings{chen2021combinatorial,
  title={Combinatorial semi-bandit in the non-stationary environment},
  author={Chen, Wei and Wang, Liwei and Zhao, Haoyu and Zheng, Kai},
  booktitle={Uncertainty in Artificial Intelligence},
  year={2021}
}

@inproceedings{cheung2019learning,
  title={Learning to optimize under non-stationarity},
  author={Cheung, Wang Chi and Simchi-Levi, David and Zhu, Ruihao},
  booktitle={The 22nd International Conference on Artificial Intelligence and Statistics},
  year={2019}
}

@article{russac2019weighted,
  title={Weighted linear bandits for non-stationary environments},
  author={Russac, Yoan and Vernade, Claire and Capp{\'e}, Olivier},
  journal={Advances in Neural Information Processing Systems},
  year={2019}
}

@inproceedings{zhao2020simple,
  title={A simple approach for non-stationary linear bandits},
  author={Zhao, Peng and Zhang, Lijun and Jiang, Yuan and Zhou, Zhi-Hua},
  booktitle={International Conference on Artificial Intelligence and Statistics},
  year={2020}
}

@article{russac2020algorithms,
  title={Algorithms for non-stationary generalized linear bandits},
  author={Russac, Yoan and Capp{\'e}, Olivier and Garivier, Aur{\'e}lien},
  journal={arXiv preprint arXiv:2003.10113},
  year={2020}
}

@article{mao2020model,
  title={Model-free non-stationary rl: Near-optimal regret and applications in multi-agent rl and inventory control},
  author={Mao, Weichao and Zhang, Kaiqing and Zhu, Ruihao and Simchi-Levi, David and Ba{\c{s}}ar, Tamer},
  journal={arXiv preprint arXiv:2010.03161},
  year={2020}
}

@article{touati2020efficient,
  title={Efficient learning in non-stationary linear Markov decision processes},
  author={Touati, Ahmed and Vincent, Pascal},
  journal={arXiv preprint arXiv:2010.12870},
  year={2020}
}

@inproceedings{wei2021non,
  title={Non-stationary reinforcement learning without prior knowledge: An optimal black-box approach},
  author={Wei, Chen-Yu and Luo, Haipeng},
  booktitle={Conference On Learning Theory (COLT)},
  year={2021}
}

@inproceedings{suk2022tracking,
  title={Tracking Most Significant Arm Switches in Bandits},
  author={Suk, Joe and Kpotufe, Samory},
  booktitle={Conference On Learning Theory (COLT)},
  year={2022}
}

@article{abbasi2022new,
  title={A new look at dynamic regret for non-stationary stochastic bandits},
  author={Abbasi-Yadkori, Yasin and Gyorgy, Andras and Lazic, Nevena},
  journal={arXiv preprint arXiv:2201.06532},
  year={2022}
}

@book{belzunce2015introduction,
  title={An introduction to stochastic orders},
  author={Belzunce, Felix and Riquelme, Carolina Martinez and Mulero, Julio},
  publisher={Academic press},
  year={2015}
}

@article{robbins1952some,
  title={Some aspects of the sequential design of experiments},
  author={Robbins, Herbert},
  year={1952}
}

@article{maddison2014sampling,
  title={A* sampling},
  author={Maddison, Chris J and Tarlow, Daniel and Minka, Tom},
  journal={Advances in neural information processing systems},
  year={2014}
}

\appendix

\section{Analysis: Omitted proofs}
\label{apx:analysis}

\restatesandwich*
\begin{proof} Let us begin with a reminder of the utility representation of categorical distributions.

\vspace{3mm}

\noindent \underline{\it Utility representation of categorical distributions:} Let $\epsilon_1, \dots, \epsilon_m$ be $m$ random variables following a Gumbel(0,1) distribution. Let $\Bar{u}_1, \dots, \Bar{u}_m$ be $m$ constants. Define the random variables $(u_i)_{i \in [m]}$ such that $u_i =\Bar{u}_i+\epsilon_i$ for every $i \in [m]$. We refer to these variables as the utilities. Then it is well known (see for example \cite{maddison2014sampling}) that the argmax of the utilities follows a categorical distribution such that, $$\mathbb{P}\left( \argmax_{k} u_k = j \right) = \frac{e^{\Bar{u}_j}}{\sum_{k=1}^m e^{\Bar{u}_k}}.$$ Here and in the reminder of the paper, since we are dealing with continuous random variables, the argmax will almost surely contain a single element. We will therefore, and without loss of generality, use the notation ``$\argmax u_k = j$" to refer to ``$\argmax u_k = \{j\}$".

\noindent \underline{\it Characterization of the distribution of $X$ in terms of utilities:} Let $(\epsilon_u(m))_{m \geq s\;,\; u \in S \cup \{0\}}$ be i.i.d. Gumbel(0,1) variables, and for every $m \geq s$, let $$Z_m:= \argmax_{u \in S \cup \{0\}} \; \left\{\; \log\left(\omega_u(\min\{m, t\})\right) + \epsilon_u(m)\;\right\},$$ where $\omega_0(k) = 1$ for every $m \geq s$ and where we adopt for the rest of the paper the conventions: $\log(0)=-\infty$ and $\mathbb{P}(-\infty \geq a) = \mathbb{P}(-\infty \geq -\infty) = 0$ for any $a \in \mathbb{R}$. Then:
\begin{claim}
    $X$ has the same distribution as the random variable, $$ \Bar{X} = \sum_{m=s}^{\argmin_{r \geq s} \{Z_r = 0\}-1} \bm{1}\{Z_m=j\}$$
\end{claim}

\begin{proof} First of all, notice that $X$ is equal to the random sum $$\sum_{m=s}^{\argmin_{r \geq s} \{j(r) = 0\}-1} \bm{1}\{j(m)=j\}$$ conditioned on $s^{(t)}_j(k) = s$ and $S(s)=S$. We prove the claim in two steps. First, we show that $\argmin_{r \geq s} \{Z_r = 0\}$ has the same distribution as $\argmin_{r \geq s} \{j(r) = 0\}$ conditioned on $s^{(t)}_j(k) = s$ and $S(s)=S$. Then we prove that for every integer $b \geq s$, the r.v. $\sum_{m=s}^{b-1} \bm{1}\{Z_m=j\}$ conditioned on $\argmin_{r \geq s} \{Z_r = 0\} = b$ has the same distribution as $\sum_{m=s}^{b-1} \bm{1}\{j(m)=j\}$ conditioned on $\argmin_{r \geq s} \{j(r) = 0\} = b$ and $s^{(t)}_j(k) = s$ and $S(s)=S$. The claim follows easily from these two facts.
\paragraph{Step 1.} For every integer $b \geq s$, we have,
\begin{align*}
    &\mathbb{P}\left(\left. \argmin_{r \geq s} \{j(r) = 0\} = b \right| s^{(t)}_j(k) = s, S(s)=S\right)\\
    =\; 
    &\mathbb{P}\left(j(b) = 0, j(b-1) \neq 0, \dots, j(s) \neq 0 \left| s^{(t)}_j(k) = s, S(s)=S\right.\right)\\\\
    =\; 
    &\mathbb{P}\left(j(b) = 0 \left| j(b-1) \neq 0, \dots, j(s) \neq 0,  s^{(t)}_j(k) = s, S(s)=S\right.\right) \times\\
    &\mathbb{P}\left(j(b-1) \neq 0 \left| j(b-2) \neq 0, \dots, j(s) \neq 0,  s^{(t)}_j(k) = s, S(s)=S\right.\right)\times \\
    &\dots\\
    &\mathbb{P}\left(j(s) \neq 0 \left|  s^{(t)}_j(k) = s, S(s)=S\right.\right)\\\\
    =\; 
    &\mathbb{P}\left(j(b) = 0 \left| S(b)=S\right.\right) \mathbb{P}\left(j(b-1) \neq 0 \left| S(b-1)=S\right.\right) \dots
    \mathbb{P}\left(j(s) \neq 0 \left| S(s)=S\right.\right)\\
    =\;
    &\frac{1}{1+\sum_{r \in S} \omega_r(\min\{b,t\})}\times \left(1- \frac{1}{1+\sum_{r \in S} \omega_r(\min\{b-1,t\})}\right) \dots \left(1- \frac{1}{1+\sum_{r \in S} \omega_r(\min\{s,t\})}\right) \\
    =\;
    &\mathbb{P}\left(Z_b = 0\right)\mathbb{P}\left(Z_{b-1} \neq 0\right) \dots \mathbb{P}\left(Z_{s} \neq 0\right)\\
    =\;
    &\mathbb{P}\left(Z_b = 0, Z_{b-1} \neq 0, \dots, Z_{s} \neq 0\right)\\
    =\;
    &\mathbb{P}\left(\argmin_{r \geq s} \{Z_r = 0\} = b\right).
\end{align*}
 The third equality follows from the fact that for every $r \geq s$ we have that $\{j(r) \neq 0, \dots, j(s) \neq 0,  s^{(t)}_j(k) = s, S(s)=S\}$ implies $\{S(r+1)=S\}$ and conditioned on $\{S(r+1)=S\}$ the random variable $j(r+1)$ is independent of all the past time steps before $r$. The fifth equality follows from the utility characterization of categorical distributions, and finally the sixth equality follows from the independence of $(Z_r)_{r \geq s}$ (since they are constructed using independent Gumbel variables).
 
\paragraph{Step 2.} Fix an integer $b \geq s$ and let $m \neq m' \in [s, b-1]$ and $j' \in S \cup \{0\}$, we have,

\begin{align*}
    &\mathbb{P}\left(j(m)=j \left|j(m')=j', \argmin_{r \geq s} \{j(r) = 0\} = b, s^{(t)}_j(k) = s, S(s)=S\right.\right)\\
    =\; 
    &\mathbb{P}\left(j(m)=j \left|j(m) \neq 0, S(m)=S\right.\right)\\
    =\; 
    &\frac{\mathbb{P}\left(j(m)=j, j(m) \neq 0 \left|S(m)=S\right.\right)}{\mathbb{P}\left(j(m) \neq 0 \left|S(m)=S\right.\right)}\\
    =\;
    &\frac{\mathbb{P}\left(j(m)=j\left|S(m)=S\right.\right)}{\mathbb{P}\left(j(m) \neq 0 \left|S(m)=S\right.\right)}\\
    =\;
    &\frac{\omega_j(\min \{m,t\})}{\sum_{r \in S} \omega_j(\min \{m,t\})}\\
\end{align*}
The first equality follows from the fact that $\{j(m')=j', \argmin_{r \geq s} \{j(r) = 0\} = b, s^{(t)}_j(k) = s, S(s)=S\}$ implies that $\{j(m) \neq 0, S(m)=S\}$, then conditioned on $\{j(m) \neq 0, S(m)=S\}$, the selected item $j(m)$ is independent of future rounds (the future rounds depend only on whether $j(m)=0$ or not and once we fix $j(m) \neq 0$, future rounds are independent of the actual sampled item $j(m)$) and independent of past rounds (once the assortment of $S(m)=S$ is fixed, $j(m)$ has the same distribution no matter what sample path leads to this assortment). This implies in particular that, conditioned on the event $\{\argmin_{r \geq s} \{j(r) = 0\} = b, s^{(t)}_j(k) = s, S(s)=S\}$, the variables $(\bm{1}\{j(m)=j\})_{m \in [s,b-1]}$ are independent and identically distributed following a Bernoulli distribution of mean $p = \frac{\omega_j(\min \{m,t\})}{\sum_{r \in S} \omega_j(\min \{m,t\})}$. Next, we have,
\begin{align*}
    &\mathbb{P}\left(Z_m=j \left| Z_{m'}=j', \argmin_{r \geq s} \{Z_r = 0\} = b\right.\right)\\
    =\;&
    \mathbb{P}\left(Z_m=j \left| Z_{m'}=j' , Z_s \neq 0, \dots, Z_{b-1} \neq 0, Z_{b}=0\right.\right)\\
    =\;&\mathbb{P}\left(Z_m=j \left| Z_{m} \neq 0 \right.\right)\\
    =\;&\frac{\mathbb{P}\left(Z_m=j , Z_{m} \neq 0 \right)}{\mathbb{P}\left(Z_{m} \neq 0 \right)}\\
    =\;&\frac{\mathbb{P}\left(Z_m=j\right)}{\mathbb{P}\left(Z_{m} \neq 0 \right)}\\
    =\;& \frac{\omega_j(\min \{m,t\})}{\sum_{r \in S} \omega_j(\min \{m,t\})},
\end{align*}
where the second equality follows from the fact that the variables $(Z_{m})_{m \geq s}$ are independent. This implies in particular that conditioned on $\{\argmin_{r \geq s} \{Z_r = 0\} = b\}$ the variables $(\bm{1}\{Z_m=j\})_{m \in [s,b-1]}$ are i.i.d. Bernoulli variables with mean $p = \frac{\omega_j(\min \{m,t\})}{\sum_{r \in S} \omega_j(\min \{m,t\})}$. Hence, the former and later variables have the same joint distribution implying that their respective sums have the same distribution. To conclude note that for every $a \in \mathbb{N}$ we have,
\begin{align*}
    \mathbb{P}\left(X = a\right)
    &=\mathbb{P}\left(\sum_{m=s}^{\argmin_{r \geq s} \{j(r) = 0\}-1} \bm{1}\{j(m)=j\} = a \left| s^{(t)}_j(k) = s,S(s)=S \right.\right)\\\\
    &=\sum_{b=s}^{\infty}\mathbb{P}\left(\sum_{m=s}^{b-1} \bm{1}\{j(m)=j\} = a \left| s^{(t)}_j(k) = s,S(s)=S, \argmin_{r \geq s} \{j(r) = 0\} = b \right.\right) \\
    & \quad \quad \;\; \times \mathbb{P}\left( \argmin_{r \geq s} \{j(r) = 0\} = b\left| s^{(t)}_j(k) = s,S(s)=S \right.\right)\\\\
    &=\sum_{b=s}^{\infty}\mathbb{P}\left(\sum_{m=s}^{b-1} \bm{1}\{Z_m=j\} = a \left| \argmin_{r \geq s} \{Z_r = 0\} = b \right.\right) \times \mathbb{P}\left( \argmin_{r \geq s} \{Z_r = 0\} = b\right)\\
    &= \mathbb{P}\left(\Bar{X} = a\right)
\end{align*}

\end{proof}

\noindent \underline{\it Coupling and proof of the stochastic bounds:} We distinguish two cases:

\paragraph{Case $s \geq t+1$:} In this case, the random variable $\Bar{X}$ has a geometric distribution with mean $\omega_j(t)$. In fact, for every $a \in \mathbb{N}$, we have,

\begin{align*}
    \mathbb{P}\left(\Bar{X} = a\right)
    &=\sum_{b=s+a}^{\infty}\mathbb{P}\left(\left.\sum_{m=s}^{b-1} \bm{1}\{Z_m=j\} = a \right| \argmin_{r \geq s} \{Z_r = 0\} = b \right) \mathbb{P}\left(\argmin_{r \geq s} \{Z_r = 0\} = b \right)\\
    &=\sum_{b=s+a}^{\infty} \binom{b-s}{a} p^a(1-p)^{b-s-a} \cdot (1-q)^{b-s}q\\
    &=\sum_{b=0}^{\infty} \binom{b+a}{a} p^a(1-p)^{b} \cdot (1-q)^{b+a}q\\
    &= q p^a(1-q)^{a} \sum_{b=0}^{\infty} \binom{b+a}{a} ((1-p)(1-q))^{b}\\
    &= \frac{ q \cdot p^a (1-q)^{a}}{(1-(1-q)(1-p))^{a+1}}\\
    &= \frac{q}{q+s} \left(1- \frac{q}{q+s}\right)^a,
\end{align*}
where $p = \frac{\omega_j(t)}{\sum_{r \in S}\omega_r(t)}$, $q = \frac{1}{1+\sum_{r \in S} \omega_r(t)}$ and $s = p(1-q)$. The second equality follows from the fact that $(\bm{1}\{Z_m=j\})_{m \in [s,b-1]}$ are i.i.d. Bernoulli variables with mean $p$ and that the random variable $\argmin_{r \geq s} \{Z_r = 0\}$ follows a geometric distribution with success probability $q$ by independence of $(Z_m)_{m \geq s}$. Therefore, $\Bar{X}$ follows a geometric distribution of mean 
$$
\frac{1-\left(\frac{q}{q+s}\right)}{\left(\frac{q}{q+s}\right)} = \omega_j(t).
$$
Now since a geometric random variable increases stochastically when its mean increases and because $$\mu^{-(t)}_j(\tau) \leq \omega_j(t) \leq \mu^{+(t)}_j(\tau)$$ it holds that $X^- \leq_{st} \Bar{X} \leq_{st} X^+$ and consequently that $X^- \leq_{st} X \leq_{st} X^+$.
\paragraph{Case $s \leq t$:} For this case, we define two random variables $\Bar{X}^{+}$ and $\Bar{X}^{-}$ using the same Gumbel variables we used in the definition of $Z_m$'s and show that, $$
X^{-} \leq_{st} \Bar{X}^{-} \leq_{st} X \leq_{st} \Bar{X}^{+} \leq_{st} X^{+}.$$ In particular, for every $m \geq s$, define $$A^+_m:= \argmax \;\left\{\; \log\left(\omega_j(\tau) + \delta^{(t)}_j\right) + \epsilon_j(m) \;,\; \left\{\; \log\left(\omega_u(\tau) - \delta^{(t)}_j\right)^+ + \epsilon_u(m)\;\right\}_{u \in S \backslash \{j\}} \;,\; \epsilon_0(m) \;\right\},$$ and, $$B^+_m:= \argmax \;\left\{\; \left\{\; \log\left(\omega_u(\tau) + \delta^{(t)}_j\right) + \epsilon_u(m)\;\right\}_{u \in S} \;,\; \epsilon_0(m) \;\right\},$$ Similarly, define, $$A^-_m:= \argmax \;\left\{\; \log\left(\omega_j(\tau) - \delta^{(t)}_j\right)^+ + \epsilon_j(m) \;,\; \left\{\; \log\left(\omega_u(\tau) + \delta^{(t)}_j\right) + \epsilon_u(m)\;\right\}_{u \in S \backslash \{j\}} \;,\; \epsilon_0(m) \;\right\},$$ and, $$B^-_m:= \argmax \;\left\{\; \left\{\; \log\left(\omega_u(\tau) - \delta^{(t)}_j\right)^+ + \epsilon_u(m)\;\right\}_{u \in S} \;,\; \epsilon_0(m) \;\right\},$$ Then let $$\Bar{X}^{+} = \sum_{m=s}^{\argmin_{r \geq s} \{B^{+}_r = 0\}-1} \bm{1}\{A^{+}_m=j\},$$ and, $$\Bar{X}^{-} = \sum_{m=s}^{\argmin_{r \geq s} \{B^{-}_r = 0\}-1} \bm{1}\{A^{-}_m=j\}.$$ We show that $\Bar{X}^{-} \leq \Bar{X} \leq \Bar{X}^{+}$ holds almost surely, which implies that 
$\Bar{X}^{-} \leq_{st} X \leq_{st} \Bar{X}^{+}$. We have,

\begin{align*}
    B^{+}_{r} = 0 
    &\implies
    \epsilon_0(r) \geq \log \left( \omega_u(\tau) + \delta^{(t)}_j \right) + \epsilon_u(r) \quad \forall u \in S\\
    &\implies
    \epsilon_0(r) \geq \log \left( \omega_u(\min\{r,t\})\right) + \epsilon_u(r) \quad \forall u \in S\\
    &\implies
    Z_r = 0.
\end{align*}
Since $\omega_u(\tau) + \delta^{(t)}_j \geq \omega_u(\min\{r,t\})$. Hence, $$\argmin_{r \geq s} \{B^{+}_r = 0\} \geq \argmin_{r \geq s} \{Z_r = 0\}$$ Next, for every $m \leq \argmin_{r \geq s} \{Z_r = 0\}$ we have, 

\begin{align*}
    Z_{m} = j
    &\implies
    \log \left( \omega_j(\min\{m,t\})\right) + \epsilon_j(m) \geq \log \left( \omega_u(\min\{m,t\})\right) + \epsilon_u(m) \quad \forall u \in S \cup \{0\}\\
    &\implies
    \left\{
    \begin{matrix}
    \log \left( \omega_j(\tau) + \delta^{(t)}_j\right) + \epsilon_j(m) \geq \log \left( \omega_u(\tau) - \delta^{(t)}_j\right)^+ + \epsilon_u(m) \quad \forall u \in S \backslash \{j\}\\
    \log \left(\omega_j(\tau) + \delta^{(t)}_j\right) + \epsilon_j(m) \geq \epsilon_0(m)\\
    \end{matrix}
    \right.\\
    &\implies
    A^{+}_m = j.
\end{align*}
Since, $(\omega_u(\tau) - \delta^{(t)}_j)^+ \leq \omega_u(\min\{m,t\}) \leq \omega_u(\tau) + \delta^{(t)}_j$. Hence, $$\bm{1}\{Z_m=j\} \leq \bm{1}\{A^+_m=j\}$$ From the above we conclude that, almost surely,

\begin{align*}
    \Bar{X}^{+} 
    &= \sum_{m=s}^{\argmin_{r \geq s} \{B^{+}_r = 0\}-1} \bm{1}\{A^{+}_m=j\}\\
    &\geq 
    \sum_{m=s}^{\argmin_{r \geq s} \{Z_r = 0\}-1} \bm{1}\{A^{+}_m=j\}\\
    &\geq 
    \sum_{m=s}^{\argmin_{r \geq s} \{Z_r = 0\}-1} \bm{1}\{Z_m=j\}\\
    &= \Bar{X}.
\end{align*}
By a similar argument we get that, $$\argmin_{r \geq s} \{B^{-}_r = 0\} \leq \argmin_{r \geq s} \{Z_r = 0\}$$ and that, $$\bm{1}\{Z_m=j\} \geq \bm{1}\{A^-_m=j\}$$ Which implies that $$\Bar{X}^{-} \leq \sum_{m=s}^{\argmin_{r \geq s} \{Z_r = 0\}-1} \bm{1}\{Z_m=j\} = \Bar{X}.$$

We now show that $\Bar{X}^{+} \leq_{st} X^{+}$ and $\Bar{X}^{-} \geq_{st} X^{-}$. Let us begin by showing that $\Bar{X}^{+}$ and $\Bar{X}^{-}$ are geometric random variable with means $$\Bar{\mu}^+ = (\omega_j(\tau)+\delta^{(t)}_j)\frac{1+ \sum_{r \in S} \omega_r(\tau) + \delta^{(t)}_r}{1+\omega_j(\tau) + \delta^{(t)}_j+ \sum_{r \in S \backslash \{j\}} (\omega_r(\tau) - \delta^{(t)}_r)^+}$$ and $$
\Bar{\mu}^- = (\omega_j(\tau)-\delta^{(t)}_j)^+\frac{1+ \sum_{r \in S} (\omega_r(\tau) - \delta^{(t)}_r)^+}{1+(\omega_j(\tau) - \delta^{(t)}_j)^+ + \sum_{r \in S \backslash \{j\}} (\omega_r(\tau) + \delta^{(t)}_r)}
$$ respectively. By a similar argument as in the first case we have for every $a \in \mathbb{N}$,

\begin{align*}
    \mathbb{P}\left(\Bar{X}^{+} = a\right)
    &=\mathbb{P}\left(\sum_{m=s}^{\argmin_{r \geq s} \{B^{+}_r = 0\}-1} \bm{1}\{A^{+}_m=j\} = a\right)\\
    &=\sum_{b=s+a}^{\infty}\mathbb{P}\left(\left.\sum_{m=s}^{b-1} \bm{1}\{A^{+}_m=j\} = a \right| \argmin_{r \geq s} \{B^{+}_r = 0\} = b \right) \mathbb{P}\left(\argmin_{r \geq s} \{B^{+}_r = 0\} = b \right)\\
    &=\sum_{b=s+a}^{\infty} \binom{b-s}{a} p^a(1-p)^{b-s-a} \cdot (1-q)^{b-s}q\\
    &= \frac{q}{q+s} \left(1- \frac{q}{q+s}\right)^a,
\end{align*}
where, 
\begin{align*}
    p
    &:= \mathbb{P}\left(\left. A^{+}_{1} = j \right| B^{+}_1 \neq 0 
    \right)\\
    &= \frac{\mathbb{P}\left( A^{+}_{1} = j \;,\; B^{+}_1 \neq 0 \right)}{\mathbb{P}\left(B^{+}_1 \neq 0\right)}\\
    &= \frac{\mathbb{P}\left( A^{+}_{1} = j\right)}{\mathbb{P}\left(B^{+}_1 \neq 0\right)}\\
    &= \left(\frac{\omega_j(\tau) + \delta^{(t)}_j}{1+\omega_j(\tau) + \delta^{(t)}_j+ \sum_{r \in S \backslash \{j\}} (\omega_r(\tau) - \delta^{(t)}_r)^+}\right) \cdot \left(\frac{1+\sum_{r \in S} (\omega_r(\tau) + \delta^{(t)}_r)}{\sum_{r \in S} (\omega_r(\tau) + \delta^{(t)}_r)}\right),
\end{align*}
and,
$q := \mathbb{P}\left(B^{+}_1 = 0\right) = \frac{1}{1+\sum_{r \in S} (\omega_r(\tau) + \delta^{(t)}_r)},$ and, $s = p(1-q).$ 
Hence, $\Bar{X}^{+}$ is a geometric random variable with probability of success, $$\frac{q}{q+s} = \frac{1}{1+ (\omega_j(\tau)+\delta^{(t)}_j)\frac{1+ \sum_{r \in S} \omega_r(\tau) + \delta^{(t)}_r}{1+\omega_j(\tau) + \delta^{(t)}_j+ \sum_{r \in S \backslash \{j\}} (\omega_r(\tau) - \delta^{(t)}_r)^+}},$$ and hence mean $$\Bar{\mu}^+ = (\omega_j(\tau)+\delta^{(t)}_j)\frac{1+ \sum_{r \in S} \omega_r(\tau) + \delta^{(t)}_r}{1+\omega_j(\tau) + \delta^{(t)}_j+ \sum_{r \in S \backslash \{j\}} (\omega_r(\tau) - \delta^{(t)}_r)^+}.$$ Similarly, for $\Bar{X}^{-}$, we have for every $a \in \mathbb{N}$,
\begin{align*}
    \mathbb{P}\left(\Bar{X}^{-} = a\right)
    &=\mathbb{P}\left(\sum_{m=s}^{\argmin_{r \geq s} \{B^{-}_r = 0\}-1} \bm{1}\{A^{-}_m=j\} = a\right)\\
    &=\sum_{b=s+a}^{\infty}\mathbb{P}\left(\left.\sum_{m=s}^{b-1} \bm{1}\{A^{-}_m=j\} = a \right| \argmin_{r \geq s} \{B^{-}_r = 0\} = b \right) \mathbb{P}\left(\argmin_{r \geq s} \{B^{-}_r = 0\} = b \right)\\
    &=\sum_{b=s+a}^{\infty} \binom{b-s}{a} p^a(1-p)^{b-s-a} \cdot (1-q)^{b-s}q\\
    &= \frac{q}{q+s} \left(1- \frac{q}{q+s}\right)^a,
\end{align*}
where, 
\begin{align*}
    p
    &:= \mathbb{P}\left(\left. A^{-}_{1} = j \right| B^{-}_1 \neq 0 
    \right)\\
    &= \frac{\mathbb{P}\left( A^{-}_{1} = j \;,\; B^{-}_1 \neq 0 \right)}{\mathbb{P}\left(B^{-}_1 \neq 0\right)}\\
    &= \frac{\mathbb{P}\left( A^{-}_{1} = j\right)}{\mathbb{P}\left(B^{-}_1 \neq 0\right)}\\
    &= \left(\frac{(\omega_j(\tau) - \delta^{(t)}_j)^+}{1+(\omega_j(\tau) - \delta^{(t)}_j)^+ + \sum_{r \in S \backslash \{j\}} (\omega_r(\tau) + \delta^{(t)}_r)}\right) \cdot \left(\frac{1+\sum_{r \in S} (\omega_r(\tau) - \delta^{(t)}_r)^+}{\sum_{r \in S} (\omega_r(\tau) - \delta^{(t)}_r)^+}\right),
\end{align*}
and, 
$q := \mathbb{P}\left(B^{-}_1 = 0\right) = \frac{1}{1+\sum_{r \in S} (\omega_r(\tau) - \delta^{(t)}_r)^+},$ and, $s = p(1-q).$ 
Hence, $\Bar{X}^{-}$ is a geometric random variable with probability of success, $$\frac{q}{q+s} = \frac{1}{1+ (\omega_j(\tau)-\delta^{(t)}_j)^+\frac{1+ \sum_{r \in S} (\omega_r(\tau) - \delta^{(t)}_r)^+}{1+(\omega_j(\tau) - \delta^{(t)}_j)^+ + \sum_{r \in S \backslash \{j\}} (\omega_r(\tau) + \delta^{(t)}_r)}},$$ and hence mean $$
\Bar{\mu}^- = (\omega_j(\tau)-\delta^{(t)}_j)^+\frac{1+ \sum_{r \in S} (\omega_r(\tau) - \delta^{(t)}_r)^+}{1+(\omega_j(\tau) - \delta^{(t)}_j)^+ + \sum_{r \in S \backslash \{j\}} (\omega_r(\tau) + \delta^{(t)}_r)}
$$ To finish the proof we bound on the means of $\Bar{X}^{+}$ and $\Bar{X}^{-}$ as follows:

\begin{align*}
    &\frac{1 + \delta^{(t)}}{1 - \delta^{(t)}} - \frac{1+ \sum_{r \in S} \omega_r(\tau) + \delta^{(t)}_r}{1+\omega_j(\tau) + \delta^{(t)}_j+ \sum_{r \in S \backslash \{j\}} (\omega_r(\tau) - \delta^{(t)}_r)^+}\\\\
    &\geq \frac{1 + \delta^{(t)}}{1 - \delta^{(t)}} - \frac{1+ \sum_{r \in S} \omega_r(\tau) + \delta^{(t)}_r}{1+\omega_j(\tau) + \delta^{(t)}_j+ \sum_{r \in S \backslash \{j\}} (\omega_r(\tau) - \delta^{(t)}_r)}\\\\
    &= 
    \frac{1}{(1 - \delta^{(t)})(1+\omega_j(\tau) + \delta^{(t)}_j+ \sum_{r \in S \backslash \{j\}} (\omega_r(\tau) - \delta^{(t)}_r))} \left[
    \left(1+\omega_j(\tau) + \delta^{(t)}_j+ \sum_{r \in S \backslash \{j\}} \omega_r(\tau) - \sum_{r \in S \backslash \{j\}} \delta^{(t)}_r\right) \right.\\
    &+\left(\delta^{(t)}+\delta^{(t)}\omega_j(\tau) + \delta^{(t)}\delta^{(t)}_j+ \delta^{(t)}\sum_{r \in S \backslash \{j\}} \omega_r(\tau) - \delta^{(t)}\sum_{r \in S\backslash \{j\}}\delta^{(t)}_r\right)\\
    &\left.+\left(-1 - \sum_{r \in S} \omega_r(\tau) - \sum_{r \in S} \delta^{(t)}_r\right)+ \left(\delta^{(t)} + \delta^{(t)}\sum_{r \in S} \omega_r(\tau) + \delta^{(t)}\sum_{r \in S}\delta^{(t)}_r\right) \right]\\\\
    &= \frac{-2 \sum_{r \in S\backslash \{j\}}\delta^{(t)}_r + 2 \delta^{(t)} + 2 \delta^{(t)}\sum_{r \in S} \omega_r(\tau) + 2 \delta^{(t)}\delta^{(t)}_j}{(1 - \delta^{(t)})(1+\omega_j(\tau) + \delta^{(t)}_j+ \sum_{r \in S \backslash \{j\}} (\omega_r(\tau) - \delta^{(t)}_r))} \geq 0,
\end{align*}
where the last inequality follows from the fact that $\delta^{(t)} \geq \sum_{r \in S \backslash \{j\}} \delta^{(t)}_r$ as $|S| \leq K$ when $s \leq t$. Hence, $$ \Bar{\mu}^+ \leq \mu^{+(t)}_j(\tau).$$ Since a geometric variable increases in stochastic order when the its mean increases we have that $\Bar{X}^{+} \leq_{st} X^{+}$. Similarly, we have,

\begin{align*}
    &\frac{1 - \delta^{(t)}}{1 + \delta^{(t)}} - \frac{1+ \sum_{r \in S} (\omega_r(\tau) - \delta^{(t)}_r)^+}{1+(\omega_j(\tau) - \delta^{(t)}_j)^+ + \sum_{r \in S \backslash \{j\}} (\omega_r(\tau) + \delta^{(t)}_r)}\\\\
    &\leq \frac{1 - \delta^{(t)}}{1 + \delta^{(t)}} - \frac{1+ \sum_{r \in S} (\omega_r(\tau) - \delta^{(t)}_r)}{1+\omega_j(\tau) + \sum_{r \in S \backslash \{j\}} (\omega_r(\tau) + \delta^{(t)}_r)}\\\\
    &= 
    \frac{1}{(1 + \delta^{(t)})(1+\omega_j(\tau) + \sum_{r \in S \backslash \{j\}} (\omega_r(\tau) + \delta^{(t)}_r))} \left[
    \left(1+\omega_j(\tau) + \sum_{r \in S \backslash \{j\}} \omega_r(\tau) + \sum_{r \in S \backslash \{j\}} \delta^{(t)}_r\right) \right.\\
    &+\left(-\delta^{(t)}-\delta^{(t)}\omega_j(\tau) - \delta^{(t)}\sum_{r \in S \backslash \{j\}} \omega_r(\tau) - \delta^{(t)}\sum_{r \in S\backslash \{j\}}\delta^{(t)}_r\right)\\
    &\left.+\left(-1 - \sum_{r \in S} \omega_r(\tau) + \sum_{r \in S} \delta^{(t)}_r\right)+ \left(-\delta^{(t)} - \delta^{(t)}\sum_{r \in S} \omega_r(\tau) + \delta^{(t)}\sum_{r \in S}\delta^{(t)}_r\right) \right]\\\\
    &= \frac{
    \sum_{r \in S\backslash \{j\}} \delta^{(t)}_r + \sum_{r \in S} \delta^{(t)}_r - 2 \delta^{(t)} - 2\delta^{(t)} \sum_{r \in S} \omega_r(\tau) + \delta^{(t)}\delta^{(t)}_j
    }{(1 + \delta^{(t)})(1+\omega_j(\tau) + \sum_{r \in S \backslash \{j\}} (\omega_r(\tau) + \delta^{(t)}_r))}\\
    &\leq \frac{
    2(\sum_{r \in S} \delta^{(t)}_r -  \delta^{(t)}) + ( \delta^{(t)}\delta^{(t)}_j -\delta^{(t)}_j)}{(1 + \delta^{(t)})(1+\omega_j(\tau) + \sum_{r \in S \backslash \{j\}} (\omega_r(\tau) + \delta^{(t)}_r))}
    \leq 0.
\end{align*}
The last inequality follows from the fact that $\delta^{(t)} \geq \sum_{r \in S} \delta^{(t)}_r$ as $|S| \leq K$ when $s \leq t$ and that $\delta^{(t)} \leq \frac{1}{2} < 1$. Hence, $$\mu^{-(t)}_j(\tau) \leq \Bar{\mu}^-.$$ Since a geometric variable increases in stochastic order when the its mean increases it holds that $X^{-} \leq_{st} \Bar{X}^{-}$.

\end{proof}

\restatemgf*
\begin{proof}
Let $t \in [1,T]$ such that $\delta^{(t)} \leq \frac{1}{2}$ and $j \in [N]$. Let $s \geq 1$, $S \in \Gamma(s)$, $k \geq 1$ and $\tau \in [1,t]$. Let $X$ denote the random variable $\zeta^{(t)}_j(k)$ conditioned on $s^{(t)}_j(k) = s $ and $ S(s)=S$. For every $0 \leq \lambda < \log \left( 1 + \frac{1}{\mu^{+(t)}_j(\tau)} \right)$, by Lemma \ref{lem:sandwich}, $X \leq_{st} X^{+}$ where $X^{+}$ is a geometric random variable with mean $\mu^{+(t)}_j(\tau)$. This implies in particular that $X\cdot \bm{1}\{X \leq M\} \leq_{st} X^{+}$ for every $M>0$. A classical result for stochastic ordering states that for a non-decreasing function $\phi$, when $Z \leq_{st} Y$ and both $\mathbb{E}(\phi(Z))$ and $\mathbb{E}(\phi(Y))$ exist it holds that $\mathbb{E}(\phi(Z)) \leq \mathbb{E}(\phi(Y))$ (see for example Theorem 2.2.5 of \cite{belzunce2015introduction}). Hence, $$\mathbb{E}\left(e^{\lambda X\cdot \bm{1}\{X \leq M\}}\right) \leq \mathbb{E}\left(e^{\lambda X^{+}}\right)$$ implying that $\mathbb{E}\left(e^{\lambda X}\right)$ exists and that $$\mathbb{E}\left(e^{\lambda X}\right) \leq \mathbb{E}\left(e^{\lambda X^{+}}\right) = \frac{1}{1-\mu^{+(t)}_j(\tau) \cdot (e^\lambda - 1)}$$ by a monotone convergence theorem. We conclude on the first inequality by noticing that,
\begin{align*}
    \mathbb{E}\left(\left. e^{\lambda \zeta^{(t)}_j(k)} \right| s^{(t)}_j(k) = s \;,\; S(s)=S\right)
    &= \mathbb{E}\left(e^{\lambda X}\right).
\end{align*}

Similarly, for every $\lambda \leq 0$, by Lemma \ref{lem:sandwich}, $X \geq_{st} X^{-}$ where $X^{-}$ is a geometric random variable with mean $\mu^{-(t)}_j(\tau)$ implying that $\lambda X\cdot \bm{1}\{X \leq M\} \leq_{st} \lambda X^{-}$ for every $M>0$ and that $$\mathbb{E}\left(e^{\lambda X\cdot \bm{1}\{X \leq M\}}\right) \leq \mathbb{E}\left(e^{\lambda X^{-}}\right)$$ and finally that $$\mathbb{E}\left(e^{\lambda X}\right) \leq \mathbb{E}\left(e^{\lambda X^{-}}\right) = \frac{1}{1-\mu^{-(t)}_j(\tau) \cdot (e^\lambda - 1)}.$$ We conclude on the second inequality by noticing that,
\begin{align*}
    \mathbb{E}\left(\left. e^{\lambda \zeta^{(t)}_j(k)} \right| s^{(t)}_j(k) = s \;,\; S(s)=S\right)
    &= \mathbb{E}\left(e^{\lambda X}\right).
\end{align*}

\end{proof}

\restateconcentration*
\begin{proof}
The first inequality is trivial for $\eta \geq 1$, hence, we only focus on the case of $0 < \eta < 1$. We have, for every $\lambda\leq 0$,
\begin{align*}
    &\mathbb{P}\left(\frac{1}{k}\sum_{i=1}^{k} \zeta^{(t)}_j(i) < (1-\eta)\mu^{-(t)}_j(\tau)\right)\\ 
    &= \mathbb{P}\left(e^{\lambda \sum_{i=1}^{k} \zeta^{(t)}_j(i)} > e^{\lambda k (1-\eta)\mu^{-(t)}_j(\tau)}\right)\\
    &\leq e^{-\lambda k (1-\eta)\mu^{-(t)}_j(\tau)} \mathbb{E}\left( \prod_{i=1}^k e^{\lambda\zeta^{(t)}_j(i)} \right)\\
    &= e^{-\lambda k (1-\eta)\mu^{-(t)}_j(\tau)} \mathbb{E}\left(
    \mathbb{E}\left(\left. e^{\lambda \zeta^{(t)}_j(k)} \right| s^{(t)}_j(k) \;,\; S\left(s^{(t)}_j(k)\right) \;,\; \zeta^{(t)}_j(1), \dots, \zeta^{(t)}_j(k-1) \right) \prod_{i=1}^{k-1} e^{\lambda \zeta^{(t)}_j(i)}
    \right)\\
    &= e^{-\lambda k (1-\eta)\mu^{-(t)}_j(\tau)} \mathbb{E}\left(
    \mathbb{E}\left(\left. e^{\lambda \zeta^{(t)}_j(k)} \right| s^{(t)}_j(k) \;,\; S\left(s^{(t)}_j(k)\right) \right) \prod_{i=1}^{k-1} e^{\lambda \zeta^{(t)}_j(i)}
    \right)\\
    &\leq e^{-\lambda k (1-\eta)\mu^{-(t)}_j(\tau)} \cdot  \mathbb{E}\left( \left(\frac{1}{1- \mu^{-(t)}_j(\tau) \cdot (e^{\lambda}-1)}\right) \prod_{i=1}^{k-1} e^{\lambda \zeta^{(t)}_j(i)} \right)\\
    &\dots\\
    &\leq e^{-\lambda k (1-\eta)\mu^{-(t)}_j(\tau)} \cdot \left(\frac{1}{1- \mu^{-(t)}_j(\tau) \cdot (e^{\lambda}-1)}\right)^{k},
\end{align*}
where the first inequality follows by Markov's inequality and the third equality holds since conditioned on the starting time of the $k$-th epoch $s^{(t)}_j(k)$ and on the assortment $S\left(s^{(t)}_j(k)\right)$, the random variable $\zeta^{(t)}_j(k)$ is independent of all past rounds and in particular of $\zeta^{(t)}_j(1), \dots, \zeta^{(t)}_j(k-1)$. Take $\lambda = -\log \left(1+\frac{\eta}{(1-\eta)(1+\mu^{-(t)}_j(\tau))}\right)$, we have, 

\begin{align*}
    e^{-\lambda k (1-\eta)\mu^{-(t)}_j(\tau)} 
    &= 
    \left(1+\frac{\eta}{(1-\eta)(1+\mu^{-(t)}_j(\tau))}\right)^{k(1-\eta)\mu^{-(t)}_j(\tau)}\\
    &\leq \left(1+\frac{\eta}{1+\mu^{-(t)}_j(\tau)}\right)^{k\mu^{-(t)}_j(\tau)}
    \\ 
    &=e^{k\mu^{-(t)}_j(\tau)\log \left(1+\frac{\eta}{1+\mu^{-(t)}_j(\tau)}\right)}\\
    &\leq e^{k\mu^{-(t)}_j(\tau) \left(\frac{\eta}{1+\mu^{-(t)}_j(\tau)} -\frac{1}{2} \left(\frac{\eta}{1+\mu^{-(t)}_j(\tau)}\right)^2 + \frac{1}{3} \left(\frac{\eta}{1+\mu^{-(t)}_j(\tau)}\right)^3\right)}\\
    &\leq e^{k\mu^{-(t)}_j(\tau) \left(\frac{\eta}{1+\mu^{-(t)}_j(\tau)} -\frac{1}{6} \left(\frac{\eta}{1+\mu^{-(t)}_j(\tau)}\right)^2\right)}\\
    &\leq e^{\frac{k\mu^{-(t)}_j(\tau)\eta}{1+\mu^{-(t)}_j(\tau)}} \cdot e^{-\frac{k\mu^{-(t)}_j(\tau)\eta^2}{24}}.
\end{align*}
The first inequality follows from the fact that $(1+x)^r \leq 1+rx$ for every $x \geq -1$ and $r \in (0,1)$. The third inequality follows from the inequality $\log(1+x) \leq x - \frac{x^2}{2} + \frac{x^3}{3}$ for every $x > -1$. The fourth inequality holds because $\frac{\eta}{1+\mu^{-(t)}_j(\tau)} \leq 1$. And the last inequality holds because $\mu^{-(t)}_j(\tau) \leq 1$. Next we have,
\begin{align*}
    \left(\frac{1}{1-\mu^{-(t)}_j(\tau) \cdot (e^{\lambda}-1)}\right)^k
    &= \left(1-\frac{\eta \mu^{-(t)}_j(\tau)}{1+\mu^{-(t)}_j(\tau)}\right)^k\\
    &= e^{k \log \left(1-\frac{\eta \mu^{-(t)}_j(\tau)}{1+\mu^{-(t)}_j(\tau)}\right)}\\
    &\leq e^{-\frac{k \mu^{-(t)}_j(\tau)\eta}{1+\mu^{-(t)}_j(\tau)}},
\end{align*}
where the last inequality follows from the fact that $\log(1+x) \leq x$ for every $x > -1$. Hence,
\begin{align*}
    &\mathbb{P}\left(\frac{1}{k}\sum_{i=1}^{k} \zeta^{(t)}_j(i) < (1-\eta)\mu^{-(t)}_j(\tau)\right)
    \leq e^{-\frac{k\mu^{-(t)}_j(\tau)\eta^2}{24}}
\end{align*}
Now for the second inequality, we have, for every $0 \leq \lambda < \log \left( 1 + \frac{1}{\mu^{+(t)}_j(\tau)} \right)$,
\begin{align*}
    &\mathbb{P}\left(\frac{1}{k}\sum_{i=1}^{k} \zeta^{(t)}_j(i) > (1+\eta)\mu^{+(t)}_j(\tau)\right)\\ 
    &= \mathbb{P}\left(e^{\lambda \sum_{i=1}^{k} \zeta^{(t)}_j(i)} > e^{\lambda k (1+\eta)\mu^{+(t)}_j(\tau)}\right)\\
    &\leq e^{-\lambda k (1+\eta)\mu^{+(t)}_j(\tau)} \mathbb{E}\left( \prod_{i=1}^k e^{\lambda\zeta^{(t)}_j(i)} \right)\\
    &= e^{-\lambda k (1+\eta)\mu^{+(t)}_j(\tau)} \mathbb{E}\left(
    \mathbb{E}\left(\left. e^{\lambda \zeta^{(t)}_j(k)} \right| s^{(t)}_j(k) \;,\; S\left(s^{(t)}_j(k)\right) \;,\; \zeta^{(t)}_j(1), \dots, \zeta^{(t)}_j(k-1) \right) \prod_{i=1}^{k-1} e^{\lambda \zeta^{(t)}_j(i)}
    \right)\\
    &= e^{-\lambda k (1+\eta)\mu^{+(t)}_j(\tau)} \mathbb{E}\left(
    \mathbb{E}\left(\left. e^{\lambda \zeta^{(t)}_j(k)} \right| s^{(t)}_j(k) \;,\; S\left(s^{(t)}_j(k)\right) \right) \prod_{i=1}^{k-1} e^{\lambda \zeta^{(t)}_j(i)}
    \right)\\
    &\leq e^{-\lambda k (1+\eta)\mu^{+(t)}_j(\tau)} \cdot  \mathbb{E}\left( \left(\frac{1}{1-\mu^{+(t)}_j(\tau) \cdot (e^{\lambda}-1)}\right) \prod_{i=1}^{k-1} e^{\lambda \zeta^{(t)}_j(i)} \right)\\
    &\dots\\
    &\leq e^{-\lambda k (1+\eta)\mu^{+(t)}_j(\tau)} \cdot \left(\frac{1}{1- \mu^{+(t)}_j(\tau)\cdot (e^{\lambda}-1)}\right)^{k}\\
\end{align*}

For $\lambda = \log \left(1+\frac{\eta}{1+\mu^{+(t)}_j(\tau)(1+\eta)}\right)$, we have, 

\begin{align*}
    e^{-\lambda k (1+\eta)\mu^{+(t)}_j(\tau)} 
    &= 
    e^{k(1+\eta)\mu^{+(t)}_j(\tau) \log \left(1-\frac{\eta}{(1+\eta)(1+\mu^{+(t)}_j(\tau))}\right)}\\
    \\
    &\leq e^{k(1+\eta)\mu^{+(t)}_j(\tau) \left(-\frac{\eta}{(1+\eta)(1+\mu^{+(t)}_j(\tau))} - \frac{1}{2}\left(\frac{\eta}{(1+\eta)(1+\mu^{+(t)}_j(\tau))}\right)^2\right)}\\
    &\leq e^{-\frac{k\mu^{+(t)}_j(\tau)\eta}{1+\mu^{+(t)}_j(\tau)}} \cdot e^{-\frac{k\mu^{+(t)}_j(\tau)\eta^2}{98(1+\eta)}},
\end{align*}
where the first inequality follows by the fact that $\log(1-x) \leq -x -\frac{x^2}{2}$ for every $x \in (0,1)$. For the last inequality, recall that $\mu^{+(t)}_j(\tau) = \left(\omega_j(\tau) +\delta^{(t)}_j\right)\frac{1+ \delta^{(t)}}{1 - \delta^{(t)}}$, note that $\delta^{(t)} \leq \frac{1}{2}$ implying $\frac{1+ \delta^{(t)}}{1 - \delta^{(t)}} \leq 3$, that $ \delta_j^{(t)} \leq \delta^{(t)} \leq \frac{1}{2}$ and that $\omega_j(\tau) \leq 1$, hence $\mu^{+(t)}_j(\tau) \leq 6$. Next, we have,

\begin{align*}
    \left(\frac{1}{1-\mu^{+(t)}_j(\tau)\cdot (e^{\lambda}-1)}\right)^k
    &= \left(1+\frac{\eta \mu^{+(t)}_j(\tau)}{1+\mu^{+(t)}_j(\tau)}\right)^k\\
    &= e^{k \log \left(1+\frac{\eta \mu^{+(t)}_j(\tau)}{1+\mu^{+(t)}_j(\tau)}\right)}\\
    &\leq e^{\frac{k \mu^{+(t)}_j(\tau)\eta}{1+\mu^{+(t)}_j(\tau)}}\\
\end{align*}
By the fact that $\log(1+x) \leq x$ for every $x > -1$. Hence,
\begin{align*}
    &\mathbb{P}\left(\frac{1}{k}\sum_{i=1}^{k} \zeta^{(t)}_j(i) > (1+\eta)\mu^{+(t)}_j(\tau)\right)
    \leq e^{-\frac{k\mu^{+(t)}_j(\tau)\eta^2}{98(1+\eta)}}
\end{align*}
To conclude the proof we show that for every $\eta > 0$, it holds that $\min \{\eta, \eta^2\} \leq \frac{2\eta^2}{1+\eta}$. In fact, if $0 < \eta<1$ then $\frac{2\eta^2}{1+\eta} \geq \frac{2\eta^2}{2} = \eta^2 = \min \{\eta, \eta^2\}$. If $\eta \geq 1$ then $\frac{2\eta^2}{1+\eta} = 2\eta\frac{\eta}{1+\eta} \geq 2\eta\cdot\frac{1}{2} = \eta = \min \{\eta, \eta^2\}$.
\end{proof}

\restateconcentrationUCB*
\begin{proof} Let $\Bar{\zeta}_j(t) := \frac{1}{\tilde{n}_j(t)} \sum_{i=1}^{n_j(t)} \zeta_j^{(t)}(i)$. Then $\Bar{\zeta}_j(t) = \Bar{\omega}_j(t)$, this is because the first $n_j(t)$ epochs all ended before $t-1$ and hence for every $i \leq n_j(t)$ it holds that $\zeta_j^{(t)}(i) = \Tilde{\omega}_j^{(t)}(i)$. Next, we have,

\begin{align*}
    &\mathbb{P}\left(\Hat{\omega}_j(t) < \mu^{-(t)}_j(\tau)\right)\\
    &=\mathbb{P}\left(\Bar{\omega}_j(t) + \sqrt{\frac{192 \Bar{\omega}_j(t) \log NT}{\tilde{n}_j(t)}} + \frac{192\log NT}{\tilde{n}_j(t)} < \mu^{-(t)}_j(\tau)\right)\\
    &=\mathbb{P}\left(\Bar{\zeta}_j(t) + \sqrt{\frac{192 \Bar{\zeta}_j(t) \log NT}{\tilde{n}_j(t)}} + \frac{192\log NT}{\tilde{n}_j(t)} < \mu^{-(t)}_j(\tau)\right)\\
    &=\sum_{k=0}^t \mathbb{P}\left(\left\{\Bar{\zeta}_j(t) + \sqrt{\frac{192 \Bar{\zeta}_j(t) \log NT}{\tilde{n}_j(t)}} + \frac{192\log NT}{\tilde{n}_j(t)} < \mu^{-(t)}_j(\tau)\right\} \; \cap \; \left\{ n_j(t) = k \right\} \right)\\\\
    &\leq \mathbb{P}\left(192\log NT < \mu^{-(t)}_j(\tau)\right)
    \\&\quad +
    \sum_{k=1}^t \mathbb{P}\left(\frac{1}{k} \sum_{i=1}^{k} \zeta_j^{(t)}(i) + \sqrt{\frac{192 \left( \frac{1}{k} \sum_{i=1}^{k} \zeta_j^{(t)}(i) \right) \log NT}{k}} + \frac{192\log NT}{k} < \mu^{-(t)}_j(\tau)\right)\\\\
    &=
    \sum_{k=1}^t \mathbb{P}\left(\frac{1}{k} \sum_{i=1}^{k} \zeta_j^{(t)}(i) + \sqrt{\frac{192 \left( \frac{1}{k} \sum_{i=1}^{k} \zeta_j^{(t)}(i) \right) \log NT}{k}} + \frac{192\log NT}{k} < \mu^{-(t)}_j(\tau)\right).
\end{align*}
In the inequality, we separate the cases $k=0$ and $k \geq 1$ then in the following equality use the fact that $\mathbb{P}\left(192\log NT < \mu^{-(t)}_j(\tau)\right) = 0$ since  $\mu^{-(t)}_j(\tau) \leq 1$. Now, let $k \in [1,t]$ and let us bound the term $$\mathbb{P}\left(\frac{1}{k} \sum_{i=1}^{k} \zeta_j^{(t)}(i) + \sqrt{\frac{192 \left( \frac{1}{k} \sum_{i=1}^{k} \zeta_j^{(t)}(i) \right) \log NT}{k}} + \frac{192\log NT}{k} < \mu^{-(t)}_j(\tau)\right).$$ First of all, taking $\eta = \sqrt{\frac{96 \log NT}{k\mu^{-(t)}_j(\tau)}}$ in the first concentration inequality of Lemma \ref{lem:concentrationZeta} gives,
\begin{align*}
    \mathbb{P}\left( 
    \frac{1}{k} \sum_{i=1}^{k} \zeta_j^{(t)}(i) + \sqrt{\frac{96 \mu^{-(t)}_j(\tau) \log NT}{k}} < \mu^{-(t)}_j(\tau)
    \right) \leq \frac{1}{(NT)^{4}}
\end{align*}
Next, taking $\eta = \frac{1}{2} + \frac{192 \log NT}{k\mu^{-(t)}_j(\tau)}$ in the same concentration inequality gives,

\begin{align*}
    \mathbb{P}\left( 
    \frac{1}{k}\sum_{i=1}^{k} \zeta_j^{(t)}(i) < \frac{1}{2}\mu^{-(t)}_j(\tau) - \frac{192 \log NT}{k}
    \right) 
    &\leq e^{-\frac{k\mu^{-(t)}_j(\tau)}{24} \left(\frac{1}{2} + \frac{192 \log NT}{k\mu^{-(t)}_j(\tau)}\right)^2}\\
    &\leq e^{-\frac{k\mu^{-(t)}_j(\tau)}{24} \left(\frac{1}{4} + \frac{192^2 \log^2 NT}{(k\mu^{-(t)}_j(\tau))^2}\right)}\\
    &= e^{-\frac{1}{96} \left(k\mu^{-(t)}_j(\tau) + \frac{4\times 192^2 \log^2 NT}{k\mu^{-(t)}_j(\tau)}\right)}\\
    &\leq e^{-\frac{2\times 192\log NT}{96}}= \frac{1}{(NT)^{4}}
\end{align*}
The last inequality follows from the fact that $(x+\frac{a}{x}) \geq \sqrt{a}$ for every $x, a > 0$. Finally we have for every $k \geq 1$,

\begin{align*}
    &\mathbb{P}\left(\frac{1}{k} \sum_{i=1}^{k} \zeta_j^{(t)}(i) + \sqrt{\frac{192 \left( \frac{1}{k} \sum_{i=1}^{k} \zeta_j^{(t)}(i) \right) \log NT}{k}} + \frac{192\log NT}{k} < \mu^{-(t)}_j(\tau)\right)\\
    &\leq
    \mathbb{P}\left(\left\{\frac{1}{k} \sum_{i=1}^{k} \zeta_j^{(t)}(i) + \sqrt{\frac{192 \left( \frac{1}{k} \sum_{i=1}^{k} \zeta_j^{(t)}(i) \right) \log NT}{k}} + \frac{192\log NT}{k} < \mu^{-(t)}_j(\tau)\right\} 
    \right.
    \\
    &\left. \quad \quad \cap\; \left\{\sum_{i=1}^{k} \zeta_j^{(t)}(i) \geq \left(\frac{1}{2}\mu^{-(t)}_j(\tau) - \frac{192 \log NT}{k}\right)^+\right\}\right) + \frac{1}{(NT)^{4}}\\\\
    &\leq
    \mathbb{P}\left(\frac{1}{k} \sum_{i=1}^{k} \zeta_j^{(t)}(i) + \sqrt{\frac{192 \left(\frac{1}{2}\mu^{-(t)}_j(\tau) - \frac{192 \log NT}{k} \right)^+ \log NT}{k}} + \frac{192\log NT}{k} < \mu^{-(t)}_j(\tau)\right) + \frac{1}{(NT)^{4}}\\ 
    &=
    \mathbb{P}\left(\frac{1}{k} \sum_{i=1}^{k} \zeta_j^{(t)}(i) + \sqrt{\left(\frac{96 \mu^{-(t)}_j(\tau) \log NT}{k} - \frac{192^2 \log^2 (NT)}{k^2}\right)^+} + \frac{192\log NT}{k} < \mu^{-(t)}_j(\tau)\right) + \frac{1}{(NT)^{4}}\\
    &\leq
    \mathbb{P}\left(\frac{1}{k} \sum_{i=1}^{k} \zeta_j^{(t)}(i) + \sqrt{\left(\frac{96 \mu^{-(t)}_j(\tau) \log NT}{k} - \frac{192^2 \log^2 (NT)}{k^2}\right)^+ + \frac{192^2\log^2 NT}{k^2}} < \mu^{-(t)}_j(\tau)\right) + \frac{1}{(NT)^{4}}\\
    &\leq
    \mathbb{P}\left(\frac{1}{k} \sum_{i=1}^{k} \zeta_j^{(t)}(i) + \sqrt{\frac{96 \mu^{-(t)}_j(\tau) \log NT}{k}} < \mu^{-(t)}_j(\tau)\right) + \frac{1}{(NT)^{4}}\\ 
    &\leq \frac{2}{(NT)^{4}},
\end{align*}
where the third inequality follows by $\sqrt{a+b} \leq \sqrt{a}+\sqrt{b}$ for every $a,b\geq 0$. Hence,

\begin{align*}
    \mathbb{P}\left(\Hat{\omega}_j(t) < \mu^{-(t)}_j(\tau)\right)
    \leq \frac{2}{NT^3}
\end{align*}
Now, for the second concentration bound, we have,

\resizebox{1\linewidth}{!}{
  \begin{minipage}{\linewidth}
  \begin{align*}
    &\mathbb{P}\left(\Hat{\omega}_j(t) > \mu^{+(t)}_j(\tau) + \sqrt{\frac{2266 \mu^{+(t)}_j(\tau) \log NT}{\tilde{n}_j(t)}} + \frac{1364\log NT}{\tilde{n}_j(t)}\right)\\ 
    &= \mathbb{P}\left(\Bar{\omega}_j(t) + \sqrt{\frac{192 \Bar{\omega}_j(t) \log NT}{\tilde{n}_j(t)}} + \frac{192\log NT}{\tilde{n}_j(t)} > \mu^{+(t)}_j(\tau) + \sqrt{\frac{2266 \mu^{+(t)}_j(\tau) \log NT}{\tilde{n}_j(t)}} + \frac{1364\log NT}{\tilde{n}_j(t)}\right)\\
    &= \mathbb{P}\left(\Bar{\zeta}_j(t) + \sqrt{\frac{192 \Bar{\zeta}_j(t) \log NT}{\tilde{n}_j(t)}} > \mu^{+(t)}_j(\tau) + \sqrt{\frac{2266 \mu^{+(t)}_j(\tau) \log NT}{\tilde{n}_j(t)}} + \frac{(1364-192)\log NT}{\tilde{n}_j(t)}\right)\\
    &= \sum_{k=0}^t \mathbb{P}\left(\left\{\Bar{\zeta}_j(t) + \sqrt{\frac{192 \Bar{\zeta}_j(t) \log NT}{\tilde{n}_j(t)}} > \mu^{+(t)}_j(\tau) + \sqrt{\frac{2266 \mu^{+(t)}_j(\tau) \log NT}{\tilde{n}_j(t)}} + \frac{(1364-192)\log NT}{\tilde{n}_j(t)} \right\} \cap \left\{ \tilde{n}_j(t) = k \right\}\right)\\\\
    &\leq \mathbb{P}\left(
    \mu^{+(t)}_j(\tau) + \sqrt{2266 \mu^{+(t)}_j(\tau) \log NT} + (1364 - 192)\log NT < 0
    \right)
    \\&\quad +
    \sum_{k=1}^t \mathbb{P}\left(\frac{1}{k}\sum_{i=1}^k \zeta^{(t)}_j(i) + \sqrt{\frac{192 \left(\frac{1}{k}\sum_{i=1}^k \zeta^{(t)}_j(i)\right) \log NT}{k}} > \mu^{+(t)}_j(\tau) + \sqrt{\frac{2266 \mu^{+(t)}_j(\tau) \log NT}{k}} + \frac{(1364-192)\log NT}{k}\right)\\\\
    &= \sum_{k=1}^t \mathbb{P}\left(\frac{1}{k}\sum_{i=1}^k \zeta^{(t)}_j(i) + \sqrt{\frac{192 \left(\frac{1}{k}\sum_{i=1}^k \zeta^{(t)}_j(i)\right) \log NT}{k}} > \mu^{+(t)}_j(\tau) + \sqrt{\frac{2266 \mu^{+(t)}_j(\tau) \log NT}{k}} + \frac{(1364-192)\log NT}{k}\right).\\
    \end{align*}
  \end{minipage}
}
Let $k \in [1,t]$ and let us bound the term

\resizebox{1\linewidth}{!}{
  \begin{minipage}{\linewidth}
    \begin{align*}
        \mathbb{P}\left(\frac{1}{k}\sum_{i=1}^k \zeta^{(t)}_j(i) + \sqrt{\frac{192 \left(\frac{1}{k}\sum_{i=1}^k \zeta^{(t)}_j(i)\right) \log NT}{k}} > \mu^{+(t)}_j(\tau) + \sqrt{\frac{2266 \mu^{+(t)}_j(\tau) \log NT}{k}} + \frac{(1364-192)\log NT}{k}\right).\\
    \end{align*}
  \end{minipage}
}
Let $\eta = \sqrt{\frac{784\log NT}{k\mu^{+(t)}_j(\tau)}} + \frac{784\log NT}{k\mu^{+(t)}_j(\tau)}$. Then $$\eta \cdot k\mu^{+(t)}_j(\tau) \geq 784\log NT$$ and $$\eta^2 \cdot k\mu^{+(t)}_j(\tau) \geq 784\log NT.$$ Hence, taking $\eta$ in the second concentration inequality of Lemma \ref{lem:concentrationZeta} gives,
\begin{align*}
    \mathbb{P}\left( 
    \sum_{i=1}^{k} \zeta_j^{(t)}(i) > \sqrt{\frac{784 \mu^{+(t)}_j(\tau) \log NT}{k}} + \frac{784\log NT}{k} + \mu^{+(t)}_j(\tau)
    \right)
    &\leq \frac{1}{(NT)^{4}}.
\end{align*}
Next, taking $\eta = 1 + \frac{784 \log NT}{k\mu^{+(t)}_j(\tau)}$ in the same concentration inequality gives,

\begin{align*}
    \mathbb{P}\left( 
    \sum_{i=1}^{k} \zeta_j^{(t)}(i) > 2\mu^{+(t)}_j(\tau) + \frac{784 \log NT}{k}
    \right) 
    &\leq e^{-\frac{k\mu^{+(t)}_j(\tau)}{196} \left(1 + \frac{784 \log NT}{k\mu^{+(t)}_j(\tau)}\right)}\\
    &\leq e^{-\frac{784 \log NT}{196}} = \frac{1}{(NT)^{4}}
\end{align*}
Finally, we have for every $k \geq 1$,

\resizebox{1\linewidth}{!}{
  \begin{minipage}{\linewidth}
    \begin{align*}
        &\mathbb{P}\left(\frac{1}{k}\sum_{i=1}^k \zeta^{(t)}_j(i) + \sqrt{\frac{192 \left(\frac{1}{k}\sum_{i=1}^k \zeta^{(t)}_j(i)\right) \log NT}{k}} > \mu^{+(t)}_j(\tau) + \sqrt{\frac{2266 \mu^{+(t)}_j(\tau) \log NT}{k}} + \frac{(1364-192)\log NT}{k}\right)\\
    &\leq 
    \mathbb{P}\left(\left\{\frac{1}{k}\sum_{i=1}^k \zeta^{(t)}_j(i) + \sqrt{\frac{192 \left(\frac{1}{k}\sum_{i=1}^k \zeta^{(t)}_j(i)\right) \log NT}{k}} > \mu^{+(t)}_j(\tau) + \sqrt{\frac{2266 \mu^{+(t)}_j(\tau) \log NT}{k}} + \frac{(1364-192)\log NT}{k}\right\} \right.\\\\
    &\quad \quad \quad \cap \left.\left\{\frac{1}{k}
    \sum_{i=1}^{k} \zeta_j^{(t)}(i) \leq 2\mu^{+(t)}_j(\tau) + \frac{784 \log NT}{k}
    \right\}  
    \right) + \frac{1}{(NT)^{4}}\\\\
    &\leq 
    \mathbb{P}\left(\frac{1}{k}\sum_{i=1}^k \zeta^{(t)}_j(i) + \sqrt{\frac{192 \left(
    2\mu^{+(t)}_j(\tau) + \frac{784 \log NT}{k}
    \right) \log NT}{k}}> \mu^{+(t)}_j(\tau) + \sqrt{\frac{2266 \mu^{+(t)}_j(\tau) \log NT}{k}} + \frac{(1364-192)\log NT}{k} 
    \right) + \frac{1}{(NT)^{4}}\\
    &= 
    \mathbb{P}\left(\frac{1}{k}\sum_{i=1}^k \zeta^{(t)}_j(i) + \sqrt{\frac{384
    \mu^{+(t)}_j(\tau) \log NT}{k} + \frac{150528 \log^2 NT
    }{k^2}}> \mu^{+(t)}_j(\tau) + \sqrt{\frac{2266 \mu^{+(t)}_j(\tau) \log NT}{k}} + \frac{(1364-192)\log NT}{k} 
    \right) + \frac{1}{(NT)^{4}}\\
    &\leq 
    \mathbb{P}\left(\frac{1}{k}\sum_{i=1}^k \zeta^{(t)}_j(i) + \sqrt{\frac{384
    \mu^{+(t)}_j(\tau) \log NT}{k}} + \sqrt{\frac{150528 \log^2 NT
    }{k^2}}> \mu^{+(t)}_j(\tau) + \sqrt{\frac{2266 \mu^{+(t)}_j(\tau) \log NT}{k}} + \frac{(1364-192)\log NT}{k} 
    \right) + \frac{1}{(NT)^{4}}\\
    &= 
    \mathbb{P}\left(\frac{1}{k}\sum_{i=1}^k \zeta^{(t)}_j(i) > \frac{1364-192- \sqrt{150528}}{784} \frac{784\log NT}{k} + \frac{\sqrt{2266}-\sqrt{384}}{\sqrt{784}}\sqrt{\frac{784
    \mu^{+(t)}_j(\tau) \log NT}{k}}  + \mu^{+(t)}_j(\tau) 
    \right) + \frac{1}{(NT)^{4}}\\
    &\leq 
    \mathbb{P}\left(\frac{1}{k}\sum_{i=1}^k \zeta^{(t)}_j(i) > \frac{784\log NT}{k} + \sqrt{\frac{784
    \mu^{+(t)}_j(\tau) \log NT}{k}}  + \mu^{+(t)}_j(\tau) 
    \right) + \frac{1}{(NT)^{4}}\\
    &\leq \frac{2}{(NT)^{4}},\\\\
    \end{align*}
  \end{minipage}
}

\noindent where the third inequality follows by $\sqrt{a+b} \leq \sqrt{a}+\sqrt{b}$ for every $a,b \geq 0$, and the fourth inequality holds because $\frac{1364-192- \sqrt{150528}}{784} \geq 1$ and $\frac{\sqrt{2266}-\sqrt{384}}{\sqrt{784}} \geq 1$. Hence,
\begin{align*}
    \mathbb{P}\left(\Hat{\omega}_j(t) > \mu^{+(t)}_j(\tau) + \sqrt{\frac{2266 \mu^{+(t)}_j(\tau) \log NT}{\tilde{n}_j(t)}} + \frac{1364\log NT}{\tilde{n}_j(t)}\right)
    \leq \frac{2}{NT^3}
\end{align*}

\end{proof}

\restatelengthepoch*
\begin{proof}
For every $i \in [1, e^{(t)}_j]$, the length of the epoch $l_j^{(t)}(i)$ is such that $$\left| l_j^{(t)}(i) \right| \leq s^{(t)}_j(i+1) - s^{(t)}_j(i).$$ Note that the inequality might be strict if the last epoch $l_j^{(t)}\left(e^{(t)}_j\right)$ is incomplete. Hence,

\begin{align*}
    &\mathbb{P}\left(\bigcup_{\substack{1 \leq i \leq e^{(t)}_j }} \left\{\left| l_j^{(t)}(i) \right| > 5 \log NT \left(1 + \sum_{k \in S\left(s_j^{(t)}(i)\right)} \omega_k(u) + \delta^{(t)}_k\right) \right\} \right)\\\\
    &\leq \mathbb{P}\left(\bigcup_{\substack{1 \leq i \leq e^{(t)}_j }} \left\{\left(s^{(t)}_j(i+1) - s^{(t)}_j(i)\right) > 5 \log NT \left(1 + \sum_{k \in S\left(s_j^{(t)}(i)\right)} \omega_k(u) + \delta^{(t)}_k\right) \right\} \right)\\\\
    &\leq \mathbb{P}\left(\bigcup_{\substack{1 \leq i \leq T}} \left\{\left(s^{(t)}_j(i+1) - s^{(t)}_j(i)\right) > 5 \log NT \left(1 + \sum_{k \in S\left(s_j^{(t)}(i)\right)} \omega_k(u) + \delta^{(t)}_k\right) \right\} \right)\\\\
    &\leq
    \sum_{i=1}^T \mathbb{P}\left(\left(s^{(t)}_j(i+1) - s^{(t)}_j(i)\right) > 5 \log NT \left(1 + \sum_{k \in S\left(s_j^{(t)}(i)\right)} \omega_k(u) + \delta^{(t)}_k\right) \right)
\end{align*}
Let $\alpha = \left(1 + \sum_{k \in S\left(s_j^{(t)}(i)\right)} \omega_k(u) + \delta^{(t)}_k\right).$ To conclude we bound the summands in the above expression. In particular, for every $i \in [1,T]$ and $j \in [N]$ we have,

\begin{align*}
    &\mathbb{P}\left(s^{(t)}_j(i+1) - s^{(t)}_j(i) > 
    5 \log NT \cdot \alpha \right)\\\\
    &= \mathbb{E}\left( \left.
    \mathbb{P}\left[s^{(t)}_j(i+1) - s^{(t)}_j(i) > 5 \log NT \cdot \alpha \right |
    s^{(t)}_j(i), S\left(s^{(t)}_j(i)\right) \right]
    \right)\\\\
    &\leq \mathbb{E}\left( \left.
    \mathbb{P}\left[j\left( s^{(t)}_j(i) \right) \neq 0 \;,\; \dots \;,\; j\left( s_j^{(t)}(i) + \left\lfloor 5 \log NT \cdot \alpha \right\rfloor - 1 \right) \neq 0  \;\right|
    s^{(t)}_j(i), S\left(s^{(t)}_j(i)\right) \right]
    \right)\\\\
    &= \mathbb{E}\left(
    \prod_{s = s_j^{(t)}(i)}^{s_j^{(t)}(i) + \left\lfloor 5 \log NT \cdot \alpha \right\rfloor - 1} \left(1-\frac{1}{1+\sum_{k \in S\left(s^{(t)}_j(i)\right)} \omega_k(\min\{s,t\})} \right)
    \right)\\
    &\leq \mathbb{E}\left(
    \prod_{s = s_j^{(t)}(i)}^{s_j^{(t)}(i) + \left\lfloor 5 \log NT \cdot \alpha \right\rfloor - 1} \left(1-\frac{1}{1+\sum_{k \in S\left(s^{(t)}_j(i)\right)} \omega_k(u) + \delta^{(t)}_k}\right)
    \right)\\
    &= \mathbb{E}\left(
    \left(1-\frac{1}{1+\sum_{k \in S\left(s^{(t)}_j(i)\right)} \omega_k(u) + \delta^{(t)}_k}\right)^{\left\lfloor 5 \log NT \cdot \alpha \right\rfloor}
    \right)\\
    &\leq \mathbb{E}\left(
    \left(1-\frac{1}{1+\sum_{k \in S\left(s^{(t)}_j(i)\right)} \omega_k(u) + \delta^{(t)}_k}\right)^{4 \log NT \left(1 + \sum_{k \in S\left(s^{(t)}_j(i)\right)} \omega_k(u) + \delta^{(t)}_k\right)}
    \right)\\
    &= \mathbb{E}\left(
    \mathrm{exp}\left[{4 \log NT \left(1 + \sum_{k \in S\left(s^{(t)}_j(i)\right)} \omega_k(u) + \delta^{(t)}_k\right) \log{\left(1-\frac{1}{1+\sum_{k \in S\left(s^{(t)}_j(i)\right)} \omega_k(u) + \delta^{(t)}_k}\right)}}\right]
    \right)\\
    &\leq \mathbb{E}\left(
    \mathrm{exp}\left[{-4 \log NT \left(1 + \sum_{k \in S\left(s^{(t)}_j(i)\right)} \omega_k(u) + \delta^{(t)}_k\right) {\left(\frac{1}{1+\sum_{k \in S\left(s^{(t)}_j(i)\right)} \omega_k(u) + \delta^{(t)}_k}\right)}}\right]
    \right)\\
    &\leq \frac{1}{NT^4}.
\end{align*}
The second equality follows by applying the chain rule of conditional probabilities. The second inequality holds because $\omega_k(\min\{s, t\}) \leq \omega_k(u) + \delta^{(t)}_k$ for every $s$ and $u$. The third inequality holds for $NT \geq e$ (so that $\alpha \cdot \log NT \geq 1$). And finally the fourth inequality follows from the fact that $\log(1-x) \leq -x$ for every $x > -1$. The proof follows from a union bound over the $T$ summands.

\end{proof}

\restateregret*
\begin{proof} We now give the proof of Theorem \ref{thm:DReg}. Let us begin by proving that $\Delta$ is indeed a non-stationarity measure (see Definition \ref{def:near-stat}). Let $t \in [1,T-1]$ and consider some $$S^* \in \displaystyle\argmax_{S \subset [N]:|S| \leq K\\} |R(S, \bm{\omega}(t)) - R(S, \bm{\omega}(t+1))|.$$ We have,

\begin{align*}
    &\max_{S \subset [N]:|S| \leq K} |R(S, \bm{\omega}(t)) - R(S, \bm{\omega}(t+1))|\\\\
    &=
    |R(S^*, \bm{\omega}(t)) - R(S^*, \bm{\omega}(t+1))|
    \\\\ &= 
    \left|\frac{\sum_{j \in S^*}r_j \omega_j(t)}{1+\sum_{k \in S^*}\omega_k(t)} - \frac{\sum_{j \in S^*}r_j \omega_j(t+1)}{1+\sum_{k \in S^*}\omega_k(t+1)}\right|
    \\\\ &=
    \left|
    \frac{\sum_{j \in S^*}r_j (\omega_j(t)-\omega_j(t+1)) + \sum_{j,k \in S^*} r_j (\omega_j(t)\omega_k(t+1) - \omega_k(t)\omega_j(t+1))}
    {(1+\sum_{k \in S^*}\omega_k(t))(1+\sum_{k \in S^*}\omega_k(t+1))}\right|
    \\\\&=
    \left|
    \frac{\sum_{j \in S^*}r_j (\omega_j(t)-\omega_j(t+1)) + \sum_{j,k \in S^*} r_j (\omega_j(t)\omega_k(t+1) - \omega_j(t)\omega_k(t) + \omega_j(t)\omega_k(t) - \omega_k(t)\omega_j(t+1))}
    {(1+\sum_{k \in S^*}\omega_k(t))(1+\sum_{k \in S^*}\omega_k(t+1))}\right|
    \\\\
    &\leq
    \sum_{j \in S^*}\left|\omega_j(t)-\omega_j(t+1)\right| 
    + 
    \frac{\sum_{j\in S^*} r_j \omega_j(t)}
    {1+\sum_{k \in S^*}\omega_k(t)}\sum_{k\in S^*}\left|\omega_k(t+1) - \omega_k(t)\right|
    \\
    &\quad +\frac{\sum_{k\in S^*} \omega_k(t)}
    {1+\sum_{k \in S^*}\omega_k(t)}\sum_{j\in S^*}\left|\omega_j(t+1) - \omega_j(t)\right|
    \\\\
    &\leq
    3 \sup_{S \subset [N]:|S| \leq 2K} \sum_{r \in S}|\omega_r(t) - \omega_r(t+1)| \leq \Delta(t)
\end{align*}

Let us now fix a $t \in [1,T]$ such that $\sum_{\tau=1}^{t-1} \Delta(\tau) \leq \rho(t)$ and $\delta^{(t)} \leq \frac{1}{2}$ and show that conditions \eqref{eq:cond1} and \eqref{eq:cond2} hold with probability at least $1-O(\frac{1}{T})$ with $\hat{R}_\tau = R(S(\tau), \bm{\hat{\omega}}(\tau))$ for every $\tau \in [1,T]$. But before moving to conditions \eqref{eq:cond1} and \eqref{eq:cond2}, we remind the following lemma from \cite{mnlucb}  characterizing the change in the expected payoff function with respect to the attraction parameters when evaluated at an optimal solution.

\begin{restatable}{lemma}{restaterevmonotonicity}[Lemma A.3 from \cite{mnlucb} ]
\label{lem:lipschitz}
    Consider the attraction parameters $\bm{\omega}$ and let $S^*$ be the optimal assortment w.r.t $\bm{\omega}$. Then for every $\bm{\omega}' \geq \bm{\omega}$ (element-wise), it holds that $ R(S^*, \bm{\omega}') \geq R(S^*, \bm{\omega})$.
\end{restatable}
Let us now prove the conditions. First, note that $(M_t = \sum_{\tau=1}^{t} R(S(\tau), \bm{\omega}(\tau)) -r(\tau))_{1 \leq t \leq T}$ is a martingale with respect to the filtration $$\mathcal{F}_0=\{\Omega, \emptyset\} \subset \mathcal{F}_1=\sigma(M_1, S(1)) \subset \dots \subset \mathcal{F}_T=\sigma(M_T, S(T)),$$ where $\sigma(M_t, S(t))$ denotes the $\sigma$-algebra generated by the random variables $M_t$ and $S(t)$ and $\Omega$ denotes the probability space. Since $R(S(\tau), \bm{\omega}(\tau)) = \mathbb{E}(r(\tau)|S(\tau))$ for all $\tau$, the common mean of $(M_t)_{t \in [1,T]}$ is $0$. Hence, by Azuma inequality we have,

\begin{align*}
    \mathbb{P}\left(\sum_{\tau=1}^{t} R(S(\tau), \bm{\omega}(\tau)) - r(\tau) > \epsilon \right) \leq e^{-\frac{\epsilon^2}{2t}},
\end{align*}
which implies that with probability at least $1-\frac{1}{T}$ it holds that $\sum_{\tau=1}^{t} R(S(\tau), \bm{\omega}(\tau)) - r(\tau) \leq \sqrt{2t\log(T)}$. Next, let $$\omega_j^{\sf up}(t) := \frac{1+\delta^{(t)}}{1-\delta^{(t)}} \cdot \Hat{\omega}_j(t) + \delta^{(t)}_j.$$ The above fact combined with Lemmas \ref{lem:concentrationUCB} and \ref{lem:lengthepochs} imply that with probability at least $1-\frac{6}{T}$ it holds that, $$\sum_{\tau=1}^{t} R(S(\tau), \bm{\omega}(\tau)) - r(\tau) \leq \sqrt{2t\log(T)},$$ $$\left\{\omega_j(\tau) \leq \omega_j^{\sf up}(t) \quad \forall \tau \leq t \in [1,T] , j \in [N]\right\},$$ $$\left\{\Hat{\omega}_j(t) \leq \mu^{+(t)}_j(\tau)  + \sqrt{\frac{2266 \mu^{+(t)}_j(\tau)  \log NT}{\tilde{n}_j(t)}} + \frac{1364\log NT}{\tilde{n}_j(t)} \quad \forall \tau \leq t \in [1,T], j \in [N] \right\},$$ and $$\left\{ 
\left| l_j^{(t)}(i) \right| \leq 5 \log NT \left(1 + \sum_{k \in S\left(s_j^{(t)}(i)\right)} \omega_k(u) + \delta^{(t)}_k\right)
 \quad \forall i \in [1, e^{(t)}_j], u \leq t \in [1,T], j \in [N]\right\}.$$ Therefore, with probability at least $1-\frac{6}{T}$, we have:
\paragraph{Condition \eqref{eq:cond1}.}

Let $\tau \in \displaystyle\argmin_{u \in [1,t]} R(S^*(u), \bm{\omega}(u))$. We have,

\begin{align*}
    \Hat{R}_t - \min_{u \in [1,t]} R(S^*(u), \bm{\omega}(u)) 
    &= R(S(t), \bm{\Hat{\omega}}(t)) - R(S^*(\tau),\bm{\omega}(\tau))\\
    &\geq R(S^*(\tau), \bm{\Hat{\omega}}(t)) - R(S^*(\tau),\bm{\omega}(\tau))\\
    &\geq R(S^*(\tau), \bm{\Hat{\omega}}(t)) - R(S^*(\tau), \bm{\omega}^{\sf up}(t))\\
    &= \frac{\sum_{k \in S^*(\tau)} r_k \Hat{\omega}_k(t)}{1+ \sum_{k \in S^*(\tau)} \Hat{\omega}_k(t)} - \frac{\sum_{k \in S^*(\tau)} r_k \omega^{\sf up}_k(t)}{1+ \sum_{k \in S^*(\tau)} \omega^{\sf up}_k(t)}\\
    &\geq \frac{\sum_{k \in S^*(\tau)} r_k (\Hat{\omega}_k(t) - \omega^{\sf up}_k(t))}{1+ \sum_{k \in S^*(\tau)} \Hat{\omega}_k(t)}\\
    &= \frac{\frac{-2\delta^{(t)}}{1-\delta^{(t)}}\sum_{k \in S^*(\tau)} r_k \Hat{\omega}_k(t) - \sum_{k \in S^*(\tau)}r_k  \delta^{(t)}_k}{1+ \sum_{k \in S^*(\tau)} \Hat{\omega}_k(t)}\\
    &\geq - 2 \frac{\delta^{(t)}}{1-\delta^{(t)}} - \delta^{(t)}\\
    &\geq - 5 \delta^{(t)} \\
    &\geq - \sum_{u=1}^{t-1}\Delta(u).
\end{align*}
The first inequality follows from the fact that $S(t)$ is optimal for the UCB $\hat{\bm{\omega}}(t)$. The second inequality follows from the fact that $S^*(t)$ is the optimal assortment under $\bm{\omega}(t)$ by applying Lemma \ref{lem:lipschitz}. The third inequality holds because $\hat{\bm{\omega}}(t) \leq \bm{\omega}^{\sf up}(t)$ coordinate-wise. The fourth inequality is because $\sum_{k \in S^*(\tau)} \delta^{(t)}_k \leq \delta^{(t)}$ as $|S^*(\tau)| \leq K$ and $r_k \leq 1$ for all $k$. And the next inequality follows from the fact that $\delta^{(t)} \leq \frac{1}{2}$.

\paragraph{Condition \eqref{eq:cond2}.}
 
We have,

$$\frac{1}{t}\sum_{\tau=1}^{t} \Hat{R}_\tau - r(\tau) = \frac{1}{t}\left(\sum_{\tau=1}^{t} \Hat{R}_\tau - R(S(\tau), \bm{\omega}(\tau))\right) + \frac{1}{t}\left(\sum_{\tau=1}^{t} R(S(\tau), \bm{\omega}(\tau)) -r(\tau)\right).$$ The second term is upper-bounded by $\sqrt{\frac{2\log T}{t}}$. For the first term, we have for every $\tau \in [1,t]$,

\resizebox{1\linewidth}{!}{
  \begin{minipage}{\linewidth}
  \begin{align*}
    &\Hat{R}_{\tau} - R(S(\tau), \bm{\omega}(\tau))\\\\
    =\; & R(S(\tau), \bm{\Hat{\omega}}(\tau)) - R(S(\tau),\bm{\omega}(\tau))\\\\
    \leq\; & R(S(\tau), \bm{\omega}^{\sf up}(\tau)) - R(S(\tau),\bm{\omega}(\tau))\\\\
    =\; & \frac{\sum_{k \in S(\tau)} r_k \omega^{\sf up}_k(\tau)}{1+ \sum_{k \in S(\tau)} \omega^{\sf up}_k(\tau)} - \frac{\sum_{k \in S(\tau)} r_k \omega_k(\tau)}{1+ \sum_{k \in S(\tau)} \omega_k(\tau)}\\\\
    \leq\; & \frac{\sum_{k \in S(\tau)} (\omega^{\sf up}_k(\tau) - \omega_k(\tau))}{1+ \sum_{k \in S(\tau)} \omega_k(\tau)}\\\\
    =\; & \frac{\sum_{k \in S(\tau)} \left(\frac{1+\delta^{(\tau)}}{1-\delta^{(\tau)}}\Hat{\omega}_k(\tau) - \omega_k(\tau)\right)}{1+ \sum_{k \in S(\tau)} \omega_k(\tau)} 
    + 
    \frac{\sum_{k \in S(\tau)} \delta^{(\tau)}_k}{1+ \sum_{k \in S(\tau)} \omega_k(\tau)}\\\\
    \leq\; & \frac{\sum_{k \in S(\tau)} \left(\frac{1+\delta^{(\tau)}}{1-\delta^{(\tau)}}\Hat{\omega}_k(\tau) - \omega_k(\tau)\right)}{1+ \sum_{k \in S(\tau)} \omega_k(\tau)} 
    + 
    \delta^{(\tau)}
    \\\\
    \leq\; & \frac{\sum_{k \in S(\tau)} \left(\frac{1+\delta^{(\tau)}}{1-\delta^{(\tau)}} 
    \mu^{+(\tau)}_k(\tau)
    - \omega_k(\tau)\right)}{1+ \sum_{k \in S(\tau)} \omega_k(\tau)}
    +
    \frac{1+\delta^{(\tau)}}{1-\delta^{(\tau)}} \cdot \frac{\sum_{k \in S(\tau)}
    \sqrt{\frac{2266 \mu^{+(\tau)}_k(\tau) \log NT}{\tilde{n}_k(\tau)}} + \frac{1364\log NT}{\tilde{n}_k(\tau)}}{1+ \sum_{k \in S(\tau)} \omega_k(\tau)}
    + 
    \delta^{(\tau)}\\\\
    \leq\; & \frac{\sum_{k \in S(\tau)} \left(\frac{1+\delta^{(\tau)}}{1-\delta^{(\tau)}} 
    \mu^{+(\tau)}_k(\tau)
    - \omega_k(\tau)\right)}{1+ \sum_{k \in S(\tau)} \omega_k(\tau)}
    +
    3 \frac{\sum_{k \in S(\tau)}
    \sqrt{\frac{2266 \mu^{+(\tau)}_k(\tau) \log NT}{\tilde{n}_k(\tau)}} + \frac{1364\log NT}{\tilde{n}_k(\tau)}}{1+ \sum_{k \in S(\tau)} \omega_k(\tau)}
    + 
    \delta^{(\tau)}.\\\\
    \end{align*}
  \end{minipage}
}
The first inequality follows from the fact that $S(\tau)$ is the optimal assortment for $\bm{\Hat{\omega}}(\tau)$ combined with Lemma \ref{lem:lipschitz} as $\bm{\Hat{\omega}}(\tau) \leq \bm{\omega}^{\sf up}(\tau)$. The second inequality is because $\bm{\omega}(\tau) \leq \bm{\omega}^{\sf up}(\tau)$. The third inequality is because and $\sum_{k \in S(\tau)} \delta^{(\tau)}_k \leq \delta^{(\tau)}$. The fourth inequality follows from the concentration bound we supposed in the beginning. And the last inequality holds because $\delta^{(\tau)} \leq \delta^{(t)} \leq \frac{1}{2}$ implying that $\frac{1+\delta^{(\tau)}}{1-\delta^{(\tau)}} \leq 3$. We bound each term in the above separately.

For the first term, we have,

\begin{align*}
    &\quad \frac{\sum_{k \in S(\tau)} \left(\frac{1+\delta^{(\tau)}}{1-\delta^{(\tau)}} 
    \mu^{+(\tau)}_k(\tau)
    - \omega_k(\tau)\right)}{1+ \sum_{k \in S(\tau)} \omega_k(\tau)}\\\\
    &=
    \frac{\sum_{k \in S(\tau)} \left(\frac{(1+\delta^{(\tau)})^2}{(1-\delta^{(\tau)})^2} 
    - 1\right)\omega_k(\tau)}{1+ \sum_{k \in S(\tau)} \omega_k(\tau)} + \left(\frac{1+\delta^{(\tau)}}{1-\delta^{(\tau)}}\right)^2 \cdot \frac{\sum_{k \in S(\tau)} \delta^{(\tau)}_k\omega_k(\tau)}{1+ \sum_{k \in S(\tau)} \omega_k(\tau)}\\\\
    &\leq
    \frac{4\delta^{(\tau)}}{(1-\delta^{(\tau)})^2} \cdot \frac{\sum_{k \in S(\tau)}  \omega_k(\tau)}{1+ \sum_{k \in S(\tau)} \omega_k(\tau)} + 9 \cdot \frac{\sum_{k \in S(\tau)} \delta^{(\tau)}_k\omega_k(\tau)}{1+ \sum_{k \in S(\tau)} \omega_k(\tau)}\\\\
    &\leq
    \frac{4\delta^{(\tau)}}{(\frac{1}{2})^2} \cdot \frac{\sum_{k \in S(\tau)}  \omega_k(\tau)}{1+ \sum_{k \in S(\tau)} \omega_k(\tau)} + 9 \cdot \frac{\sum_{k \in S(\tau)} \delta^{(\tau)}_k\omega_k(\tau)}{1+ \sum_{k \in S(\tau)} \omega_k(\tau)}\\\\
    &\leq
    16 \delta^{(\tau)} + 9 \delta^{(\tau)} \\
    &= 25 \delta^{(\tau)}\\
    \end{align*}
The first equality is just by replacing $\mu^{+(\tau)}_k(\tau)$ by its expression. The first inequality is by the fact that $$\frac{(1+x)^2}{(1-x)^2} - 1= \frac{4x}{(1-x)^2}$$ and that $\frac{1+\delta^{(\tau)}}{1-\delta^{(\tau)}} \leq 3$. The second inequality is by $\delta^{(\tau)} \leq \frac{1}{2}$ and the next inequality is because $\sum_{k \in S(\tau)} \delta^{(\tau)}_k \leq \delta^{(\tau)}$ and $\omega_k(\tau) \leq 1$ for every $k$ and $\tau$.

Now for the second term we have,

\begin{align*}
    &\quad \; 3 \frac{\sum_{k \in S(\tau)}
    \sqrt{\frac{2266 \mu^{+(\tau)}_k(\tau) \log NT}{\tilde{n}_k(\tau)}} + \frac{1364\log NT}{\tilde{n}_k(\tau)}}{1+ \sum_{k \in S(\tau)} \omega_k(\tau)}\\
    &=
    3 \frac{\sum_{k \in S(\tau)}
    \sqrt{\frac{2266 \frac{1+\delta^{(\tau)}}{1-\delta^{(\tau)}}\left(\omega_k(\tau)+\delta^{(\tau)}_k\right) \log NT}{\tilde{n}_k(\tau)}} + \frac{1364\log NT}{\tilde{n}_k(\tau)}}{1+ \sum_{k \in S(\tau)} \omega_k(\tau)}\\
    &\leq
    3 \frac{2\sum_{k \in S(\tau)}
    \sqrt{\frac{2266 \left(\omega_k(\tau) + \delta^{(\tau)}_k \right) \log NT}{\tilde{n}_k(\tau)}} + \frac{1364\log NT}{\tilde{n}_k(\tau)}}{1+ \sum_{k \in S(\tau)} \omega_k(\tau)}\\
    &=
    6\sqrt{2266 \log NT} \cdot \frac{\sum_{k \in S(\tau)}
    \sqrt{\frac{\omega_k(\tau) + \delta^{(\tau)}_k}{\tilde{n}_k(\tau)}}}{1+ \sum_{k \in S(\tau)} \omega_k(\tau)}
    + 
    4092 \log NT \frac{\sum_{k \in S(\tau)}
    \frac{1}{\tilde{n}_k(\tau)}}{1+ \sum_{k \in S(\tau)} \omega_k(\tau)}
    \\
    &\leq
    6\sqrt{2266 \log NT} \cdot \sqrt{\frac{\sum_{k \in S(\tau)} \omega_k(\tau) + \delta^{(\tau)}_k}{1+ \sum_{k \in S(\tau)} \omega_k(\tau)} \cdot \frac{\sum_{k \in S(\tau)}
    \frac{1}{\tilde{n}_k(\tau)}}{1+ \sum_{k \in S(\tau)} \omega_k(\tau)}}
    \\
    & \quad + 4092 \log NT \frac{\sum_{k \in S(\tau)}
    \frac{1}{\tilde{n}_k(\tau)}}{1+ \sum_{k \in S(\tau)} \omega_k(\tau)}
    \\\\
    &\leq
    6\sqrt{2266 \log NT} \cdot \sqrt{\frac{\delta^{(\tau)} + \sum_{k \in S(\tau)} \omega_k(\tau)}{1+ \sum_{k \in S(\tau)} \omega_k(\tau)} \cdot \frac{\sum_{k \in S(\tau)}
    \frac{1}{\tilde{n}_k(\tau)}}{1+ \sum_{k \in S(\tau)} \omega_k(\tau)}}
    \\
    & \quad + 4092 \log NT \frac{\sum_{k \in S(\tau)}
    \frac{1}{\tilde{n}_k(\tau)}}{1+ \sum_{k \in S(\tau)} \omega_k(\tau)}
    \\\\
    &\leq
    6\sqrt{2266 \log NT} \cdot \sqrt{\frac{\sum_{k \in S(\tau)}
    \frac{1}{\tilde{n}_k(\tau)}}{1+ \sum_{k \in S(\tau)} \omega_k(\tau)}} + 4092 \log NT \frac{\sum_{k \in S(\tau)}
    \frac{1}{\tilde{n}_k(\tau)}}{1+ \sum_{k \in S(\tau)} \omega_k(\tau)}.
    \\
\end{align*}
The first inequality holds by the fact that $\frac{1+\delta^{(\tau)}}{1-\delta^{(\tau)}} \leq 4$. The second is by Cauchy-Schwarz inequality. The third inequality is because $\sum_{k \in S(\tau)} \delta^{(\tau)}_k \leq \delta^{(\tau)}$ and the last inequality is because $\delta^{(\tau)} \leq \delta^{(t)} \leq 1$. Hence,

\begin{align*}
    &\sum_{\tau=1}^t \Hat{R}_{\tau} - R(S(\tau), \bm{\omega}(\tau))\\ 
    &\leq 
    26 t \delta^{(t)} + 6\sqrt{2266 \log NT} \cdot \sum_{\tau=1}^t \sqrt{\frac{\sum_{k \in S(\tau)}
    \frac{1}{\tilde{n}_k(\tau)}}{1+ \sum_{k \in S(\tau)} \omega_k(\tau)}} + 4092 \log NT \sum_{\tau=1}^t \frac{\sum_{k \in S(\tau)}
    \frac{1}{\tilde{n}_k(\tau)}}{1+ \sum_{k \in S(\tau)} \omega_k(\tau)}\\
    &\leq
    26 t \delta^{(t)} + 6\sqrt{2266 \log NT} \cdot \sqrt{t \cdot \sum_{\tau=1}^t \frac{\sum_{k \in S(\tau)}
    \frac{1}{\tilde{n}_k(\tau)}}{1+ \sum_{k \in S(\tau)} \omega_k(\tau)}} + 4092 \log NT \sum_{\tau=1}^t \frac{\sum_{k \in S(\tau)}
    \frac{1}{\tilde{n}_k(\tau)}}{1+ \sum_{k \in S(\tau)} \omega_k(\tau)}.\\
\end{align*}
The first inequality follows from the fact that $\delta^{(\tau)} \leq \delta^{(t)}$ for $\tau \leq t$, and the second one is an application of Cauchy-Schwarz inequality. Next fix some time step $u \in [1,t]$, we have,

\begin{align*}
    \sum_{\tau=1}^t \frac{\sum_{k \in S(\tau)}
    \frac{1}{\tilde{n}_k(\tau)}}{1+ \sum_{k \in S(\tau)} \omega_k(\tau)}
    &\leq
    \sum_{\tau=1}^t \frac{\sum_{k \in S(\tau)}
    \frac{1}{\tilde{n}_k(\tau)}}{1+ \sum_{k \in S(\tau)} \omega_k(u) - \delta^{(t)}_k}\\
    &=
    \sum_{j=1}^N\sum_{i=1}^{e_j^{(t)}} \frac{\left| l_j^{(t)}(i) \right| \cdot
    \frac{1}{\max\{1, i-1\}}}{1+ \sum_{k \in S\left(s_j^{(t)}(i)\right)} \omega_k(u) - \delta^{(t)}_k}\\
    &\leq
    \sum_{j=1}^N\sum_{i=1}^{e_j^{(t)}} \frac{5\log NT \left( 1 + \sum_{k \in S\left(s_j^{(t)}(i)\right)} \omega_k(u) + \delta^{(t)}_k\right)
    \frac{1}{\max\{1, i-1\}}}{1+ \sum_{k \in S\left(s_j^{(t)}(i)\right)} \omega_k(u) - \delta^{(t)}_k}\\
\end{align*}
The first inequality follows by the fact that $1 + \sum_{k \in S(\tau)}\omega_k(\tau) \geq 1 + \sum_{k \in S(\tau)}\omega_k(u) - \delta^{(t)}_k \geq 0$ for every $\tau, u \in [1,t]$. And equality holds because of the following: the sum consists of a term $\frac{\frac{1}{\tilde{n}_k(\tau)}}{1+ \sum_{k \in S(\tau)} \omega_k(u) - \delta^{(t)}_k}$ for each time step $\tau \in [1,t]$ and each item $k$ offered at $\tau$. Rearranging these terms, for each item $j$ and each epoch $i \in [1, e^{(t)}_j]$ it was proposed at, we have a term $\frac{\frac{1}{\tilde{n}_j(\tau)}}{1+ \sum_{k \in S(\tau)} \omega_k(u) - \delta^{(t)}_k}$ for each $\tau$ belonging to epoch $l^{(t)}_j(i)$, the equality follows by noticing that the assortment offered within epoch $l^{(t)}_j(i)$ is fixed to $S(s^{(t)}_j(i))$ for all time rounds of the epoch, and that $\tilde{n}_k(\tau)$ is fixed and equal to $\frac{1}{\max\{1, i-1\}}$ for all time rounds of the epoch. The last inequality follows from the concentration bound on the length of epochs. Next, note that,

\begin{align*}
    4 - \frac{1 + \sum_{k \in S\left(s_j^{(t)}(i)\right)} \omega_k(u) + \delta^{(t)}_k}{1+ \sum_{k \in S\left(s_j^{(t)}(i)\right)} \omega_k(u) - \delta^{(t)}_k}
    &=
    \frac{3 + 3 \sum_{k \in S\left(s_j^{(t)}(i)\right)} \omega_k(u) - 5 \sum_{k \in S\left(s_j^{(t)}(i)\right)} \delta^{(t)}_k}{1+ \sum_{k \in S\left(s_j^{(t)}(i)\right)} \omega_k(u) - \delta^{(t)}_k}\\
    &\geq
    \frac{3 - 5 \delta^{(t)}}{1+ \sum_{k \in S\left(s_j^{(t)}(i)\right)} \omega_k(u) - \delta^{(t)}_k}\\
    &\geq
    \frac{3- \frac{5}{2}}{1+ \sum_{k \in S\left(s_j^{(t)}(i)\right)} \omega_k(u) - \delta^{(t)}_k} \geq 0
\end{align*}
Hence,

\begin{align*}
    \sum_{\tau=1}^t \frac{\sum_{k \in S(\tau)}
    \frac{1}{\tilde{n}_k(\tau)}}{1+ \sum_{k \in S(\tau)} \omega_k(\tau)}
    &\leq
    20\log NT \cdot \sum_{j=1}^N\sum_{i=1}^{e_j^{(t)}}
    \frac{1}{\max\{1, i-1\}}\\
    &\leq
    20\log NT \cdot (N \log T + 2)\\
    &\leq 40N \log^2 NT
\end{align*}
For $N\log T \geq 2$, note that the number of epochs $e^{(t)}_j \leq T$. Therefore,

\begin{align*}
    \frac{1}{t} \sum_{\tau=1}^t \Hat{R}(\tau) - R(\tau) 
    &\leq
    26 \delta^{(t)} + 1807 (\log{NT})^{\frac{3}{2}}\sqrt{\frac{N}{t}} + 163680 (\log{NT})^{3} \frac{N}{t}.
\end{align*}
And finally,
\begin{align*}
    \frac{1}{t}\sum_{\tau=1}^{t} \Hat{R}(\tau) - r(\tau) 
    &\leq
    26 \delta^{(t)} + 1807 (\log{NT})^{\frac{3}{2}}\sqrt{\frac{N}{t}} + 163680 (\log{NT})^{3} \frac{N}{t} + \sqrt{\frac{2\log T}{t}}\\
    &\leq \sum_{\tau=1}^{t-1}\Delta(\tau) + \rho(t)
\end{align*}
\end{proof}

\section{Lower bounds: Omitted proofs}
\label{apx:lowerbounds}

\restatedependenceL*
\noindent Our proof for Theorem \ref{thm:dependenceL} uses the results of \cite{chen2017note} for stationary environments. In this paper, Chen and Wang give an optimal lower bound of $\Tilde{\Omega}(\sqrt{NT})$ for MNL-Bandit in stationary environments. Following \cite{chen2017note}, we suppose $K \leq \frac{N}{4}$. This assumption is easily verified in practice. In online retail for example, the number of items we can display to the customers at each round is usually small compared to the total number of items.

\vspace{3mm}
\begin{proof}
    Fix $T, L$, $N$ and $K \leq \frac{N}{4}$ and let $\mathcal{A}$ be a polynomial time algorithm for MNL-Bandit in non-stationary environments. If $L=1$, one can simply choose a (stationary) adversarial instance given by Theorem 1 of \cite{chen2017note} and the result holds for this case as the number of switches in a stationary environment is $L=1$. Otherwise, if $L \geq 2$, we partition the decision horizon into $L$ intervals of $\left\lceil \frac{T}{L} \right\rceil$ time steps each (except possibly the last interval which might have less time steps). We construct an adversarial instance recursively such that $\mathcal{A}$ suffers an expected regret of at least $C \cdot \min(\sqrt{N\left\lceil \frac{T}{L} \right\rceil}, \left\lceil \frac{T}{L} \right\rceil)$ in each one of the first $L-1$ intervals for some universal constant $C > 0$. In particular, suppose we fixed the adversarial instance for the first $l$ ($\leq L-2$) intervals. For the next $\left\lceil \frac{T}{L} \right\rceil$ time steps of the $l+1$-th interval, Theorem 1 in \cite{chen2017note} implies the existence of a (stationary) adversarial instance such that algorithm $\mathcal{A}$ (more precisely the randomized policy that $\mathcal{A}$ plays in the $l+1$-th interval conditioned on the part we have fixed of our adversarial instance) suffers a regret of at least $C \cdot \min(\sqrt{N\left\lceil \frac{T}{L} \right\rceil}, \left\lceil \frac{T}{L} \right\rceil)$ over the $l+1$-th interval for some universal constant $C > 0$. We choose such instance for the $l+1$-th interval and continue our construction. We complete our instance in the last interval ($L$-th interval) arbitrarily. Our construction has at most $L$ switches and the regret suffered by algorithm $\mathcal{A}$ over the time horizon $T$ is at least $(L-1) \cdot C \min(\sqrt{N\left\lceil \frac{T}{L} \right\rceil}, \left\lceil \frac{T}{L} \right\rceil) \geq \frac{C}{2} \cdot \min(\sqrt{NLT}, T)$.
\end{proof}

\restatedependenceDelta*
\noindent To prove Theorem \ref{thm:dependenceDelta}, we divide the time horizon into windows of convenient length and use the results of \cite{chen2017note} to construct a stationary adversarial instance for each window. With a good choice of these adversarial instances, our final instance is such that $\deltaIk \geq \frac{K}{2} \cdot \deltaI$.

\vspace{3mm}
\begin{proof}
    Fix $T$, $N$, $K \leq \frac{N}{4}$ and $\Delta \in [\frac{1}{N}, \frac{T}{N}]$ and let $\mathcal{A}$ be a polynomial time algorithm for MNL-Bandit in non-stationary environments. We partition the decision horizon into $\left\lceil\frac{T}{M}\right\rceil$ intervals of length $M$ (except perhaps the last interval which might have less time steps) for some $M$ to be decided later. Let $M_l$ denote the length of the $l$-th interval and let $\epsilon \in (0,1)$. We construct an adversarial instance against algorithm $\mathcal{A}$ recursively. In particular, suppose we fixed our adversarial instance for the first $l$ ($\leq \left\lceil\frac{T}{M}\right\rceil-1$) intervals. For the $(l+1)$-th interval we do the following: for every $\eta \in (0,\frac{1}{2})$, \cite{chen2017note} shows the existence of an adversarial strategy $\pi(\eta)$ against $\mathcal{A}$ where exactly $K$ items have an attraction parameter $\frac{1+\eta}{K}$ while every other item has an attraction parameter $\frac{1}{K}$, and against which $\mathcal{A}$ collects a regret of at least 
    $$
    \frac{\eta}{9}\left(\frac{2M_{l+1}}{3} - M_{l+1} \sqrt{\frac{126 \cdot M_{l+1} \cdot \eta^2}{N}}\right)
    $$
    over the $M_{l+1}$ rounds of $l+1$-th interval. If $l$ is even, we choose $\pi(\epsilon)$ for the $l+1$-th interval and continue our construction. If $l$ is odd, we choose $\pi(\frac{\epsilon}{2})$. Then with the choices, $$\epsilon = \min \left(\frac{1}{17} \sqrt{\frac{N}{M}}, \frac{1}{2} \Delta\cdot\frac{M}{T}\right) \quad (\leq \frac{1}{2})$$ and 
    $$
    M = \left\lceil N^{\frac{1}{3}} \left(\frac{T}{\Delta}\right)^{\frac{2}{3}} \right\rceil,
    $$
    the attraction parameters in the constructed instance change at most $\left\lceil\frac{T}{M}\right\rceil - 1$ times (between the intervals), each time there is a change, at most $2K$ elements have their attraction parameter changed by at most $\frac{\epsilon}{K}$ (in absolute value). The variation $\deltaIk$ of the constructed instance is therefore at most 
    $$
    \frac{T}{M}\cdot 2K \cdot \frac{\epsilon}{K} = \frac{2T}{M}\epsilon \leq \Delta.
    $$
    Also, by construction, between every two intervals (an odd and an even interval) at least $K$ items have their attraction parameters changed by at least $\frac{\epsilon}{2}$ (in absolute value). Since the change in norm $L_{\infty}$ (largest change happening in a single parameter)  between two intervals is at most $\epsilon$ this implies that 
    $$
    \deltaIk \geq \frac{K}{2} \deltaI.
    $$
    Now, the expected regret suffered by $\mathcal{A}$ over the whole time horizon is at least,
    \begin{align*}
        \sum_{l=1}^{\left\lceil\frac{T}{M}\right\rceil} \frac{1}{9} \cdot\frac{\epsilon}{2}\left(\frac{2M_l}{3} - M_l \sqrt{\frac{126M_l\epsilon^2}{N}}\right)
        &\geq
        \sum_{l=1}^{\left\lceil\frac{T}{M}\right\rceil} \frac{1}{9} \cdot\frac{\epsilon}{2}\left(\frac{2M_l}{3} - M_l \sqrt{\frac{126M\epsilon^2}{N}}\right)\\
        &=
        \frac{1}{9} \cdot\frac{\epsilon}{2}\left(\frac{2T}{3} - T \sqrt{\frac{126M\epsilon^2}{N}}\right)\\
        &=
        \frac{\epsilon T}{18}\left(\frac{2}{3} -  \sqrt{126} \left(\epsilon \sqrt{\frac{M}{N}}\right)\right)\\
        &\geq
        \frac{\epsilon T}{18}\left(\frac{2}{3} - \frac{\sqrt{126}}{17}\right)\\
        &\geq
        C \cdot \epsilon T\\
        &\geq
        C \cdot T \cdot \min \left(\frac{1}{17} \sqrt{\frac{N}{M}}, \frac{1}{2} \Delta\cdot\frac{M}{T}\right)\\
        &\geq
        C \cdot T \cdot \min \left(\frac{1}{17} N^{\frac{1}{3}} T^{-\frac{1}{3}} \Delta^{\frac{1}{3}}, \frac{1}{2} N^{\frac{1}{3}} \Delta^{\frac{1}{3}} T^{-\frac{1}{3}}\right)\\
        &\geq
        \frac{C}{17} N^{\frac{1}{3}} T^{\frac{2}{3}} \Delta^{\frac{1}{3}}\\
    \end{align*}
    where $C > 0.0003$ is a absolute constant independent of the parameters of the problem. The first inequality holds because $M_l \leq M$ for every $l$ and the following equality is because $\sum_{l=1}^{\left\lceil\frac{T}{M}\right\rceil} M_l = T$.
\end{proof}

\end{document}